\documentclass[letterpaper]{article}

\usepackage[preprint]{aaai2027}
\usepackage[hyphens]{url}
\usepackage{graphicx}
\usepackage{natbib}
\usepackage{caption}
\usepackage{amsmath}
\usepackage{amssymb}
\usepackage{amsthm}
\usepackage{microtype}
\usepackage{algorithm}
\usepackage{algpseudocode}
\usepackage{dblfloatfix}
\usepackage{booktabs}
\usepackage{xcolor}
\usepackage{colortbl}
\definecolor{tableheader}{RGB}{246,227,223}
\definecolor{tablebody}{RGB}{242,242,242}
\definecolor{tableours}{RGB}{253,248,227}
\definecolor{annotationcyan}{RGB}{0,139,175}
\newcommand{\tabletoprule}{\specialrule{\heavyrulewidth}{0pt}{0pt}}
\newcommand{\tablemidrule}{\specialrule{\lightrulewidth}{0pt}{0pt}}
\newcommand{\tablebottomrule}{\specialrule{\heavyrulewidth}{0pt}{0pt}}
\newcommand{\tablesingleheadpad}{%
  \vrule width 0pt
  height \dimexpr\ht\strutbox+4.52pt\relax
  depth \dimexpr\dp\strutbox+4.52pt\relax}

\newcommand{\method}{SpikeOPD}
\newcommand{\spad}{SpAD}
\newcommand{\bispikclm}{BiSpikCLM}
\newcommand{\loss}{\mathcal{L}}
\newcommand{\KL}{\mathrm{KL}}
\newcommand{\E}{\mathbb{E}}
\newcommand{\V}{\mathcal{V}}
\newcommand{\pospart}[1]{\left[#1\right]_{+}}
\newtheorem{proposition}{Proposition}

\title{\method{}: Stable On-Policy Distillation for Autoregressive Spiking Language Models}
\author{%
Enqiao Lu\textsuperscript{1,2,3}\quad
Xingrui Yu\textsuperscript{2,\(\dagger\)}\quad
Yiwei Fu\textsuperscript{4}\quad
Zhenglin Wan\textsuperscript{5}\quad
Pengfei Zhou\textsuperscript{5}\\
Wangbo Zhao\textsuperscript{5}\quad
Muqing Jian\textsuperscript{1,6}\quad
Xueyi Zhang\textsuperscript{5}\quad
Yang You\textsuperscript{5}\quad
Ivor Tsang\textsuperscript{2,3}%
}
\affiliations{%
{\small
\textsuperscript{1} The Chinese University of Hong Kong, Shenzhen\quad
\textsuperscript{2} Agency for Science, Technology and Research (A*STAR), Singapore\\
\textsuperscript{3} Nanyang Technological University, Singapore\quad
\textsuperscript{4} Peking University\quad
\textsuperscript{5} National University of Singapore\quad
\textsuperscript{6} Rice University}%
}

\begin{document}
\maketitle
\begingroup
\renewcommand{\thefootnote}{\fnsymbol{footnote}}
\footnotetext[2]{Corresponding author: \texttt{yu\_xingrui@a-star.edu.sg}.}
\endgroup

\begin{abstract}
Spiking neural networks (SNNs) offer a path to energy-efficient language modeling through sparse encoding and event-driven computation, but training capable spiking language models from scratch remains difficult. A practical alternative is ANN-to-SNN migration through knowledge distillation (KD), where a pretrained artificial neural network (ANN) teacher supervises an SNN student. Existing migration approaches distill on fixed corpus prefixes, whereas autoregressive inference conditions on self-generated prefixes, creating \textbf{prefix-source mismatch}. It manifests as \textbf{output-policy mismatch} with the ANN teacher and \textbf{internal spiking-dynamics drift} between self-generated and matched corpus prefixes. On-policy distillation (OPD) offers a natural way to mitigate both manifestations by continuing teacher supervision on self-generated prefixes. We evaluate a teacher-only full-KL variant, Vanilla OPD, via a controlled stress test and observe it may suffer from delayed rollout-feedback collapse. This result shows that on-policy coverage alone does not ensure stable adaptation. Motivated by these findings, we propose \method{}, a stable on-policy distillation framework for autoregressive SNNs that learns from self-generated prefixes while maintaining rollout stability. It applies full-KL teacher correction to reduce output-policy mismatch, while matched-prefix policy anchoring constrains policy departure from the frozen reference SNN on the same prefixes. Layerwise spike regularization further limits firing-rate deviations during on-policy adaptation. Across three model scales, \method{} improves average accuracy over the corresponding KD SNNs by 0.8, 1.7, and 2.9 points at 0.125B, 0.35B, and 1.3B, respectively, while preserving their sparse-compute profiles.

\end{abstract}

\section{Introduction}

Autoregressive language generation repeatedly applies a large network to produce one token at a time, making dense activation computation and data movement persistent inference costs \citep{horowitz20141}. Spiking neural networks (SNNs) instead use sparse binary spikes to trigger computation on event-driven neuromorphic hardware \citep{davies2018loihi,roy2019towards}. This execution model is attractive for long generation, where sparse computation recurs across decoding steps. Recent work has extended SNNs from vision to language encoders and causal generators \citep{zhu2023spikegpt,lv2025spikebert,bal2024spikingbert,xing2024spikelm,xing2024spikellm,guo2026bispikclm,zhou2026winner}. Yet large generative SNNs remain difficult to train from scratch \citep{fang2021deep,zheng2021going}. We therefore ask how to transfer a pretrained ANN language model into an SNN without losing language ability or sparse activity.

Existing ANN-to-SNN migration strategies directly inherit pretrained ANN weights and then calibrate or fine-tune the resulting SNN \citep{diehl2015fast,rueckauer2017conversion,sengupta2019going,rathi2020enabling,deng2021optimal}, introduce quantization-like spike representations for language models \citep{xing2024spikelm}, or transfer ANN knowledge through distillation \citep{hinton2015distilling}. This paper focuses on distillation-based ANN-to-SNN migration, where an ANN teacher supervises an SNN student on fixed corpus prefixes. For example, \bispikclm{} is a causal SNN trained by Spike-Aware Alignment Distillation (\spad{}), which transfers embeddings, attention, features, and token targets from an ANN teacher \citep{guo2026bispikclm}. During autoregressive inference, the SNN instead conditions on self-generated prefixes, and one changed token can alter every later prefix. This prefix-source mismatch is related to exposure bias and learner-induced covariate shift \citep{bengio2015scheduled,ross2011reduction}. A generative SNN additionally exposes recurrent spiking dynamics beyond the cited dense-model formulations \citep{bellec2018long}.

We call the training-to-inference shift from corpus to self-generated prefixes prefix-source mismatch, related to exposure bias and learner-induced covariate shift \citep{bengio2015scheduled,ross2011reduction}. Paired-prefix controls isolate two manifestations under matched prompts, positions, and lengths: output-policy mismatch with the ANN teacher and internal spiking-dynamics drift. This gives our first question: \textbf{RQ1: How does prefix-source mismatch manifest as output-policy mismatch and internal spiking-dynamics drift?}

To mitigate both manifestations of prefix-source mismatch, a common choice is on-policy distillation (OPD), which continues teacher supervision on self-generated prefixes. We examine its teacher-only full-KL variant, Vanilla OPD. Because the policy being updated also generates its training prefixes, a large correction can redirect the next rollout and change the following training distribution. This feedback motivates controlling how each correction redirects later rollouts. This gives our second question: \textbf{RQ2: Can Vanilla OPD adapt without sustained degeneration of rollout statistics, which we call rollout-feedback collapse?}

We answer both questions empirically. For RQ1, paired-prefix controls show two manifestations of prefix-source mismatch: greater output-policy mismatch with the ANN teacher and internal spiking-dynamics drift across spike rates, hidden states, and membrane potentials. A separate stress test answers RQ2: Vanilla OPD preserves early task accuracy before entering delayed rollout-feedback collapse. These results shape \method{}. Full-KL teacher correction addresses output-policy mismatch, while layerwise spike regularization penalizes firing-rate deviations from a frozen reference SNN. The instability exposed by RQ2 motivates matched-prefix policy anchoring, which compares the active SNN and frozen reference SNN on the same self-generated prefixes to limit policy movement. Hidden-state and membrane-potential gaps remain diagnostics. Across eight tasks, \method{} improves the corresponding KD SNN baselines by 0.8, 1.7, and 2.9 average points at 0.125B, 0.35B, and 1.3B. Our contributions are:
\begin{itemize}
    \item \textbf{Diagnostic findings.} Paired-prefix diagnostics show that prefix-source mismatch manifests as output-policy mismatch and internal spiking-dynamics drift. Matched-seed stress tests further reveal delayed rollout-feedback collapse in Vanilla OPD.
    
    \item \textbf{Stable on-policy distillation.} We propose \method{}, combining full-KL teacher correction, matched-prefix policy anchoring, and layerwise spike regularization over self-generated prefixes.
    
    \item \textbf{Multi-scale validation.} Across 0.125B, 0.35B, and 1.3B SNNs, \method{} improves eight-task average accuracy by 0.8, 1.7, and 2.9 points, remains non-collapsed in all 10 matched-seed runs at 0.125B, and retains the analytical sparse-compute profile.
\end{itemize}

\section{Related Work}

\paragraph{Training and transferring spiking language models.}
Deep SNNs are commonly obtained through surrogate-gradient training, ANN-to-SNN conversion, or teacher-student transfer. Surrogate gradients enable temporal training \citep{neftci2019surrogate,wu2018spatio,fang2021deep,zheng2021going}, conversion calibrates inherited weights and firing rates \citep{diehl2015fast,rueckauer2017conversion,sengupta2019going,rathi2020enabling,deng2021optimal}, and distillation supervises logits, features, or attention \citep{hinton2015distilling,romero2014fitnets,zagoruyko2016paying,sanh2019distilbert,sun2019patient,jiao2020tinybert}. Spike-form attention was introduced in vision Transformers before language work progressed from distilled encoders to recurrent and larger causal generators \citep{zhou2022spikformer,yao2023spike,lv2025spikebert,bal2024spikingbert,zhu2023spikegpt,xing2024spikelm,xing2024spikellm}. BiSpikCLM combines causal spiking attention with spike-aware alignment, while winner-take-all spiking Transformers provide another causal-attention design \citep{guo2026bispikclm,zhou2026winner}. These systems obtain capable SNN checkpoints from fixed corpus prefixes but do not address prefix-source mismatch during autoregressive sampling.

\paragraph{On-policy language-model distillation.}
Sequence-level distillation, scheduled sampling, and DAgger motivate learning beyond fixed corpus prefixes \citep{kim2016sequence,bengio2015scheduled,ross2011reduction}. Recent methods vary the divergence and rollout source: MiniLLM uses reverse KL, DistiLLM mixes on- and off-policy samples, and generalized OPD applies teacher losses to student generations \citep{gu2024minillm,ko2024distillm,agarwal2024policy}. EOPD, Lightning OPD, and OPRD add entropy-aware selection, stored trajectories, or representation alignment \citep{jin2026entropy,wu2026lightning,yang2026oprd}. More broadly, KL-controlled policy updates limit departure from a reference behavior in policy optimization and language-model alignment \citep{schulman2015trust,ouyang2022training}. These objectives and diagnostics remain mainly token- or representation-level, whereas SNN updates can also alter spike activity and membrane trajectories. This distinction motivates diagnosing internal spiking-dynamics drift separately from output-policy mismatch.

\section{Empirical Analysis}

\subsection{Prefix-Source Mismatch in Knowledge Distillation}

RQ1 pairs each self-generated prefix from the KD checkpoint with a corpus prefix matched by prompt, position, and length. Figures~\ref{fig:teacher_kl} and \ref{fig:dynamics} summarize the output and internal effects. All methods use the same fixed prefix bank.

\begin{figure}[!htbp]
    \centering
    \includegraphics[width=1.0\columnwidth]{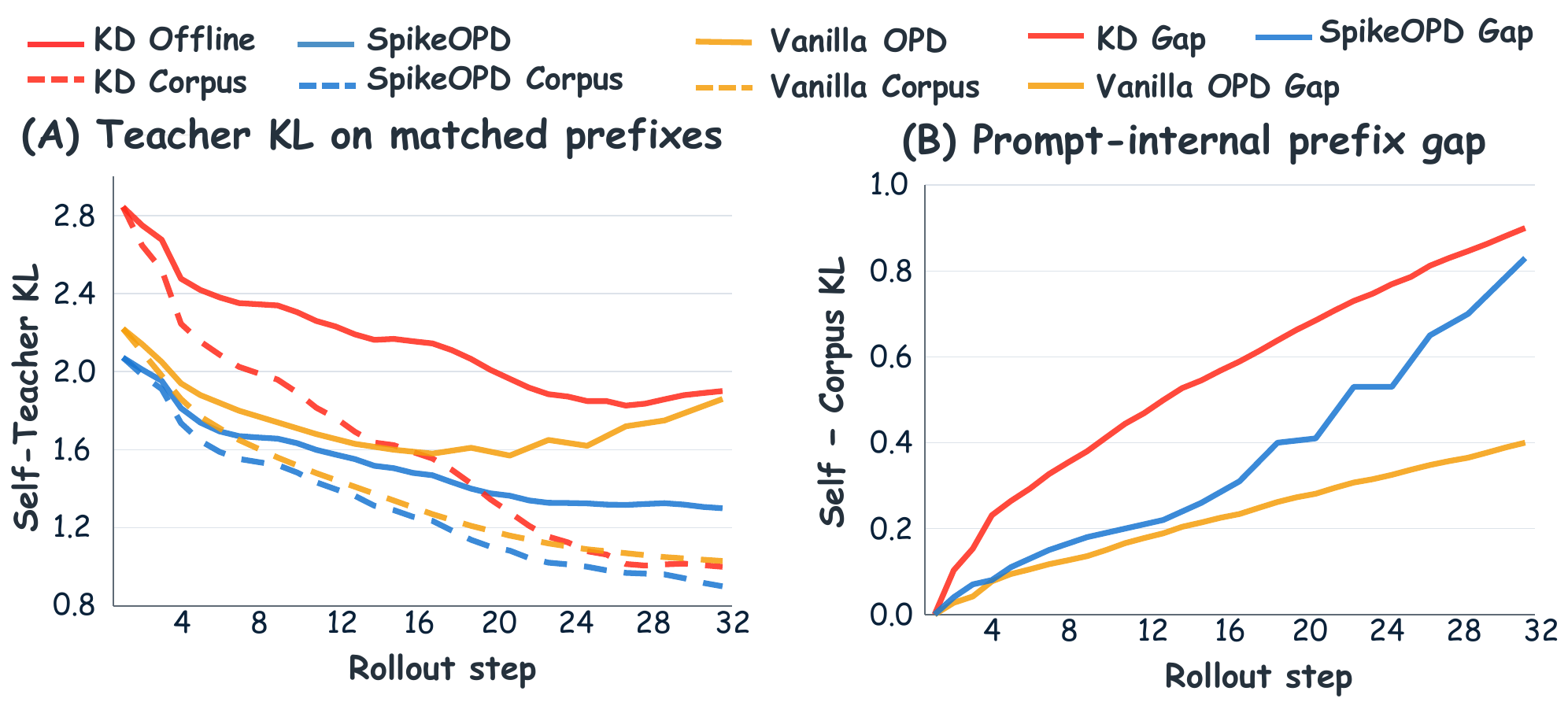}
    \caption{Paired-prefix output-policy diagnostic on 256 prompts. (A) Teacher-to-SNN KL on self-generated (solid) and corpus (dashed) prefixes. (B) Prompt-internal Self$-$Corpus gap.}
    \label{fig:teacher_kl}
\end{figure}

\begin{figure*}[!t]
    \centering
    \includegraphics[width=\textwidth]{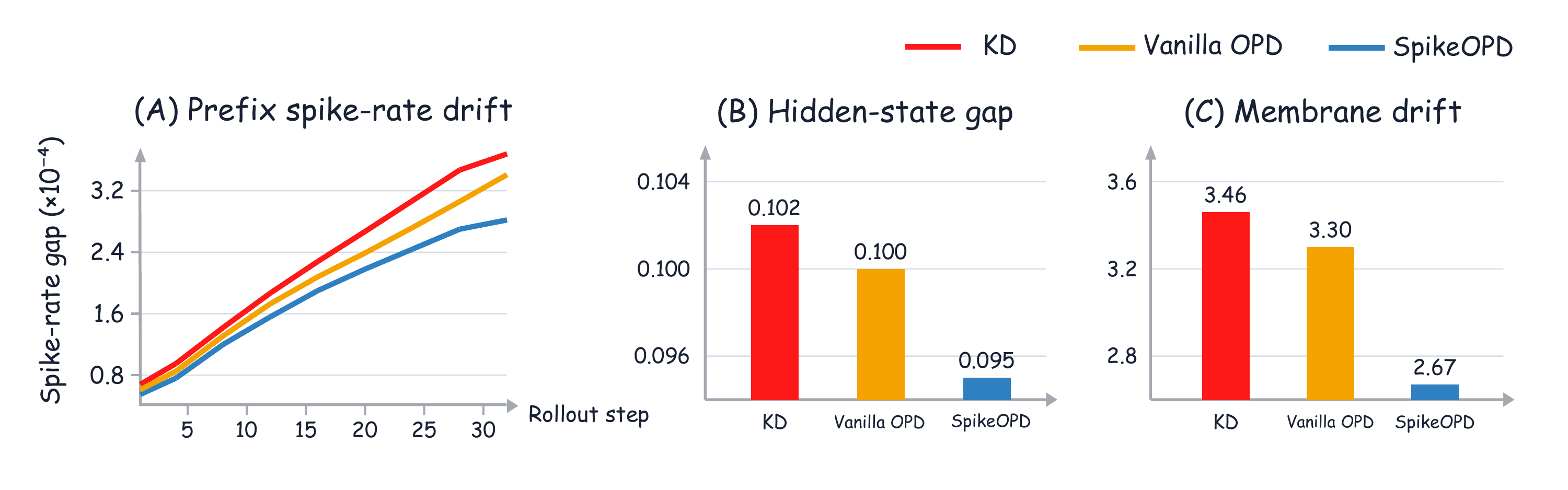}
    \caption{Internal effects under matched prompts, positions, and lengths: (A) spike-rate gap, (B) hidden-state gap, and (C) aggregate membrane-potential gap. Smaller is better.}
    \label{fig:dynamics}
\end{figure*}

\paragraph{\mbox{Output-policy mismatch.}}
For prompt $p$, position $t$, checkpoint $c$, and source $a\in\{\mathrm{self},\mathrm{corp}\}$, output-policy mismatch is the full-vocabulary teacher-to-SNN KL
\begin{equation}
\kappa_{c,p,t}^{a}=\sum_{v\in\mathcal V}p_T(v\mid s_{p,t}^{a})
\log\frac{p_T(v\mid s_{p,t}^{a})}{q_c(v\mid s_{p,t}^{a})}.
\label{eq:diagnostic_kl}
\end{equation}
Figure~\ref{fig:teacher_kl}(A) shows higher KL for KD than \method{} on both sources. Because corpus KL also varies with position, we isolate the paired effect with
\begin{equation}
G_c(t)=\frac{1}{N}\sum_{p=1}^{N}
\left(\kappa_{c,p,t}^{\mathrm{self}}-\kappa_{c,p,t}^{\mathrm{corp}}\right).
\label{eq:prefix_gap}
\end{equation}
In Figure~\ref{fig:teacher_kl}(B), $G_{\mathrm{KD}}(t)$ grows from near zero to about $0.9$, whereas \method{} remains smaller. Thus \method{} reduces both absolute teacher KL and the paired prefix-source effect.

\paragraph{Internal spiking-dynamics drift.}
For layer $l$ with $N_l$ neurons, we define the firing rate and the within-checkpoint prefix-source gap as
\begin{equation}
\begin{aligned}
r_{p,t,l}^{c,a}
&=\frac{1}{T_sN_l}\sum_{\tau,i}s_{p,t,l,i,\tau}^{c,a},\\
D_{\mathrm{spk}}(c,t)
&=\frac{1}{NL}\sum_{p,l}
\left|r_{p,t,l}^{c,\mathrm{self}}-r_{p,t,l}^{c,\mathrm{corp}}\right|.
\end{aligned}
\label{eq:spike_drift}
\end{equation}
The KD gap grows with rollout position. Vanilla OPD reduces it, and \method{} remains lower throughout, as shown in panel A of Figure~\ref{fig:dynamics}. Aggregated $L_1$ hidden-state and membrane-potential gaps fall from $0.102$ and $3.46$ for KD to $0.100$ and $3.30$ for Vanilla OPD, then to $0.095$ and $2.67$ for \method{}, as shown in panels B and C.

These controlled results answer RQ1 and motivate supervising both output behavior and SNN-specific event dynamics. We next ask whether teacher-only full-KL OPD can do so without destabilizing later rollouts.
\subsection{Vanilla OPD Is Insufficient for Stable Distillation}

Starting from the same KD checkpoint, Vanilla OPD applies only full-KL teacher supervision to self-generated prefixes. Figure~\ref{fig:teacher_kl} shows a late warning: its Self$-$Corpus gap accelerates toward KD, whereas \method{} remains lower.

Figure~\ref{fig:dynamics} provides the complementary internal view, where Vanilla OPD retains greater internal drift than \method{} in all three diagnostics under matched prefixes.
Its hidden-state and membrane-potential gaps remain at $0.100$ and $3.30$, above the $0.095$ and $2.67$ achieved by \method{}.

We stress-test 10 matched seeds using adjacent repetition and $1-\text{Dist.-4}$. Collapse begins when their 20-update means exceed $0.10$ and $0.05$, respectively. 

\begin{table}[!t]
\centering
{
\small
\renewcommand{\arraystretch}{1.2}
\setlength{\tabcolsep}{1pt}
\rowcolors{2}{tablebody}{white}
\begin{tabular}{@{}lccccc|ccc@{}}
\tabletoprule
\rowcolor{tableheader}
\tablesingleheadpad
Method & LR & Clip & Step & C/T & Onset
& Avg.\,$\uparrow$
& Adj.\,$\downarrow$
& Dist.-4\,$\uparrow$ \\
\tablemidrule

Vanilla OPD
& $\eta_0$ & - & 500 & $10/10$ & 168.5
& 34.8 & 39.1 & 83.7 \\

Small-LR
& $\eta_0/3$ & - & 1000 & $10/10$ & 527
& 34.8 & 36.0 & 85.5 \\

Clipped
& $\eta_0$ & 1.0 & 500 & $10/10$ & 327
& 34.5 & 31.5 & 87.8 \\

\rowcolor{tableours}
\method{}
& $\eta_0$ & - & 1000 & $0/10$ & -
& \textbf{35.5} & \textbf{0.1} & \textbf{99.2} \\

\tablebottomrule
\end{tabular}
}
\caption{Stability controls from one KD checkpoint. C/T reports collapsed/tested runs, and Onset averages first-collapse updates. Collapse uses 20-update Adj. and $1-\text{Dist.-4}$ thresholds of 0.10 and 0.05. Avg. is the last pre-onset eight-task average or final average.}
\label{tab:opd_multiseed}
\end{table}

\begin{figure}[!b]
    \centering
    \includegraphics[width=\columnwidth]{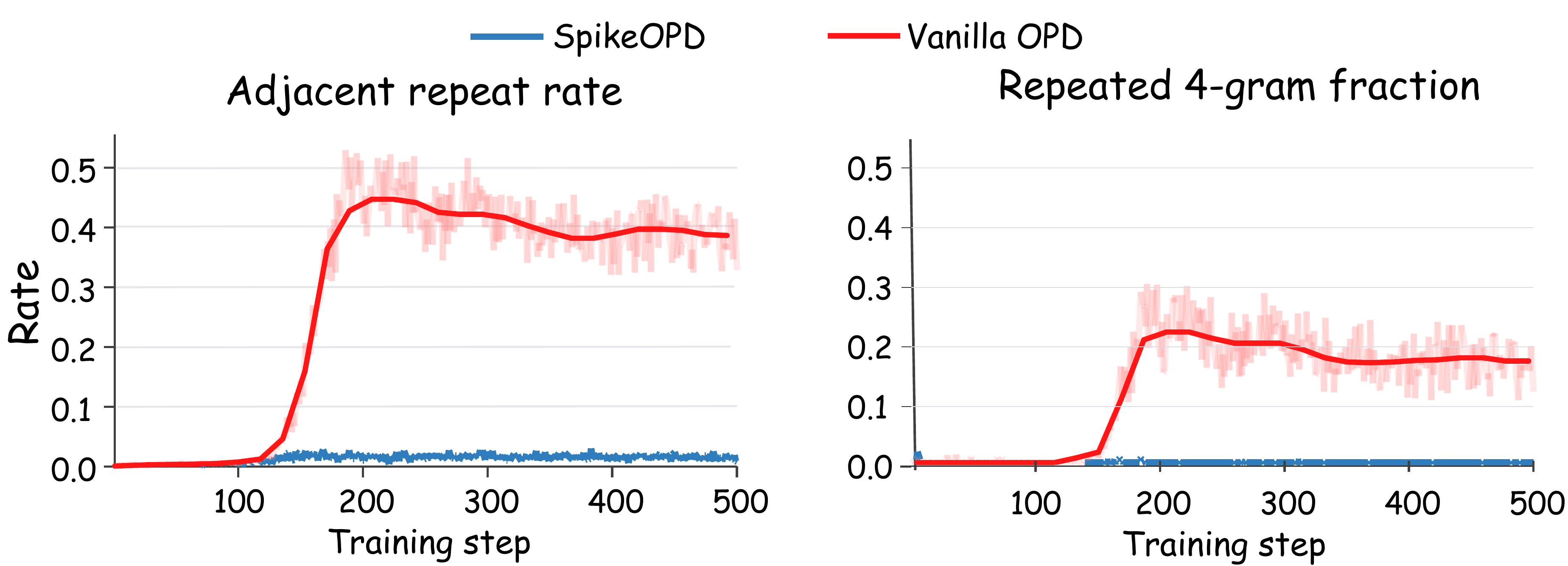}
    \caption{Delayed rollout-feedback collapse in one of 10 matched trajectories. Adjacent repetition measures equal neighboring-token pairs, while $1-\text{Dist.-4}$ measures repeated contiguous four-token sequences. Vanilla OPD becomes repetitive, whereas \method{} remains stable. Lower is better.}
    \label{fig:opd_collapse}
\end{figure}

\begin{figure*}[!t]
    \centering
    \includegraphics[width=\textwidth]{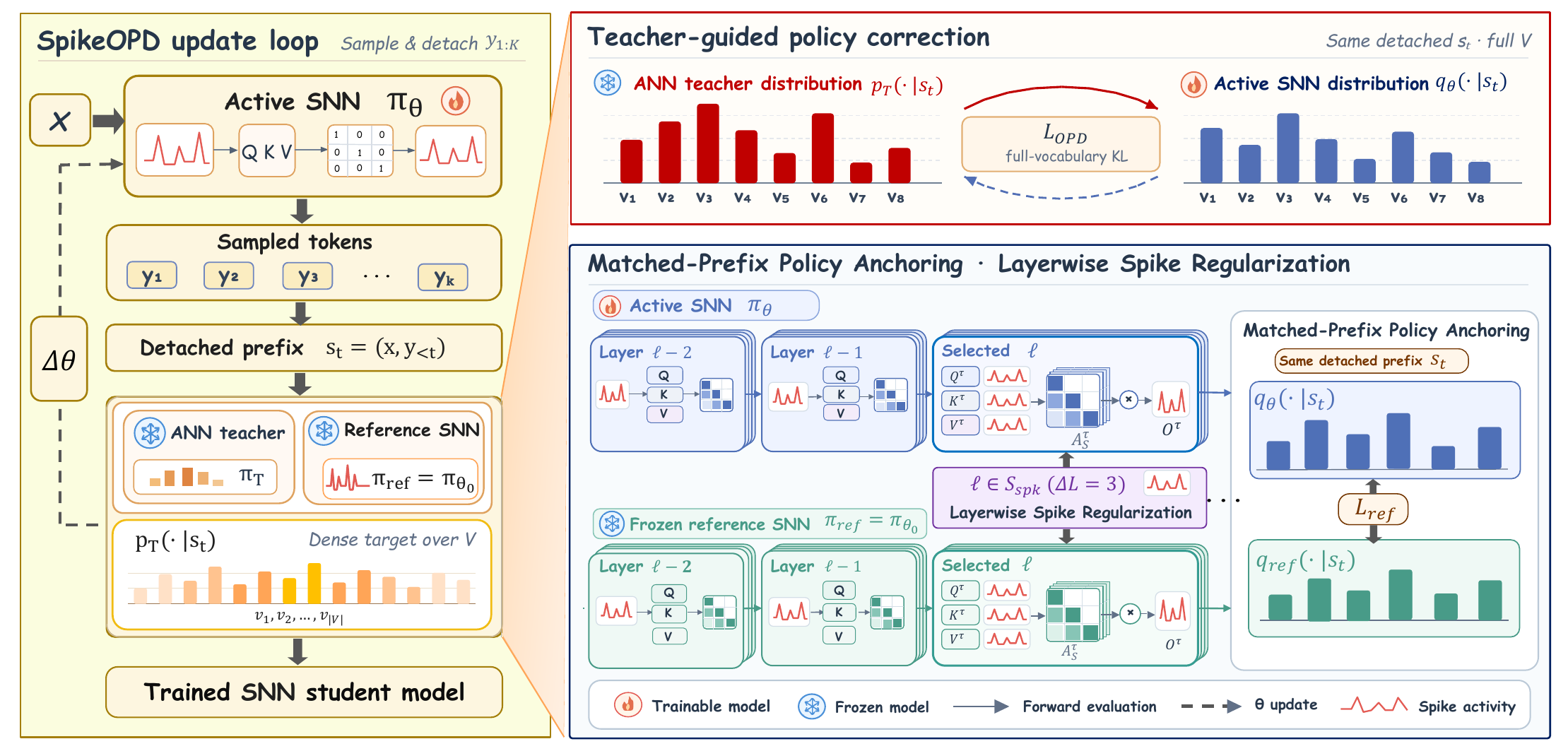}
    \caption{Overview of \method{}. On the same detached self-generated prefixes, teacher correction transfers the ANN distribution, policy anchoring constrains departure from the frozen reference SNN, and spike regularization controls selected-layer firing rates. Only $\theta$ is updated.}
    \label{fig:framework}
\end{figure*}

Figure~\ref{fig:opd_collapse} shows the delay: early updates remain close to KD, but repeated text then enters later training contexts and shifts the rollout distribution. All 10 Vanilla OPD seeds eventually collapse.

Table~\ref{tab:opd_multiseed} shows that Small-LR and clipping only postpone collapse to updates 527 and 327. Later onset is not stability over the full budget. \method{} remains stable in all 10 runs. Their failure motivates a frozen reference SNN on the same self-generated prefixes, coupled with teacher correction and spike regularization.

\section{Method}

\subsection{Preliminaries}

\paragraph{Spiking neurons and rate coding.}
The causal student follows the decoder-only \bispikclm{} architecture \citep{guo2026bispikclm}. A leaky integrate-and-fire neuron updates as
\begin{equation}
u_i^{\tau}=\alpha u_i^{\tau-1}+I_i^{\tau}-V_{\mathrm{th}}s_i^{\tau-1},
\qquad s_i^{\tau}=\mathbf{1}[u_i^{\tau}\ge V_{\mathrm{th}}].
\label{eq:lif}
\end{equation}
Here $u_i^{\tau}$, $I_i^{\tau}$, and $s_i^{\tau}$ denote the membrane potential, input current, and binary spike indicator, while $\alpha$ and $V_{\mathrm{th}}$ are the leak factor and firing threshold. The next computation receives the rate code $\bar h=T_s^{-1}\sum_\tau s^\tau$. The thresholding operation is trained with surrogate gradients \citep{neftci2019surrogate}.

\paragraph{Softmax-free causal spiking attention.}
For spike-coded $H^\tau$, \bispikclm{} forms binary queries, keys, and values:
\begin{equation}
\begin{aligned}
Q^\tau&=\mathcal S(H^\tau W_Q),\quad K^\tau=\mathcal S(H^\tau W_K),\\
V^\tau&=\mathcal S(H^\tau W_V),\\
A_S^\tau&=\mathcal S\!\left(Q^\tau(K^\tau)^\top/\sqrt d\right)\odot M_{\mathrm{causal}},\\
O^\tau&=A_S^\tau V^\tau.
\end{aligned}
\label{eq:spiking_attention}
\end{equation}
$\mathcal S$ denotes the spike function, and $M_{\mathrm{causal}}$ preserves autoregressive factorization. Binary $A_S^\tau$ routes computation through spike events. The same events determine layerwise activity, leaving firing rates free to drift under output-only matching.

\subsection{Overview of \method{}}

Offline KD creates prefix-source mismatch when an SNN conditions on self-generated rather than corpus prefixes, while direct on-policy correction can destabilize later rollouts. To address both issues, \method{} adapts a KD checkpoint using teacher-guided policy correction, matched-prefix policy anchoring, and layerwise spike regularization. As shown in Figure~\ref{fig:framework}, these components jointly correct the output policy, constrain policy departure, and stabilize internal spiking dynamics.

\subsection{Teacher-Guided Policy Correction}

Prefix-source mismatch places the SNN on prefixes absent from its offline corpus. At each update, the active SNN $\pi_\theta$ samples $K$ tokens and forms prefixes $s_t=(x,y_{<t})$. The frozen ANN teacher $\pi_T$ evaluates the same prefixes, giving $p_T(v\mid s_t)=\pi_T(v\mid s_t)$ and $q_\theta(v\mid s_t)=\pi_\theta(v\mid s_t)$. Sampled token identities are detached, while gradients pass through $q_\theta$ and the active-SNN spike paths. We transfer the teacher's full distribution with
\begin{equation}
\loss_{\mathrm{OPD}}=\frac{1}{K}\sum_{t=1}^{K}
\KL\!\left(p_T(\cdot\mid s_t)\,\|\,q_\theta(\cdot\mid s_t)\right).
\label{eq:opd}
\end{equation}
The logit gradient $q_\theta-p_T$ corrects every vocabulary coordinate and retains teacher uncertainty. It provides dense supervision on self-generated prefixes that are rare offline. Appendix~C gives the full-vocabulary objective and its gradient properties.

Full-KL coverage exposes the teacher on self-generated prefixes, but early corrections can redirect later rollouts. Vanilla OPD's delayed collapse in the controlled stress test therefore motivates an explicit policy-movement constraint.

\subsection{Matched-Prefix Policy Anchoring}

The reference SNN is a frozen copy of the initial KD checkpoint, $\pi_{\mathrm{ref}}=\pi_{\theta_0}$. Because corpus prefixes do not capture where the active policy changes, the active and reference SNNs evaluate identical self-generated prefixes. This matched-prefix design supplies a stability target that preserves transferred language ability and sparse execution, giving $q_{\mathrm{ref}}(v\mid s_t)=\pi_{\mathrm{ref}}(v\mid s_t)$ and
\begin{equation}
\loss_{\mathrm{ref}}=\frac{1}{K}\sum_{t=1}^{K}
\KL\!\left(q_\theta(\cdot\mid s_t)\,\|\,q_{\mathrm{ref}}(\cdot\mid s_t)\right).
\label{eq:reference}
\end{equation}
Here, $\loss_{\mathrm{OPD}}$ (see Eq.~\eqref{eq:opd}) aligns the active SNN with the ANN teacher, while $\loss_{\mathrm{ref}}$ anchors it to the frozen SNN on the identical self-generated prefixes without blocking teacher-guided correction.

\newcommand{\mainresultstable}{%
\begin{table*}[!t]
\centering
{
\small
\rowcolors{3}{tablebody}{white}
\begin{tabular}{llccccccccc|ccc}
\tabletoprule

\rowcolor{tableheader}
\rule{0pt}{2.8ex}
& &
\multicolumn{9}{c|}{\textbf{Zero-shot Accuracy (\%)}} &
\multicolumn{3}{c}{\textbf{Analytical Efficiency}} \\

\noalign{%
  \global\aboverulesep=0pt
  \global\belowrulesep=0pt
}
\cmidrule[0.1pt](lr){3-11}
\cmidrule[0.1pt](lr){12-14}
\noalign{%
  \global\aboverulesep=.4ex
  \global\belowrulesep=.65ex
}

\rowcolor{tableheader}
\rule{0pt}{2.8ex}
Model & Scale & ARC-e & ARC-c & WG & BQ & PIQA & HS &
OBQA & HQA & Avg.\smash{$\uparrow$}
& OPs\smash{$\downarrow$}
& Rate\smash{$\downarrow$}
& Energy\smash{$\downarrow$} \\

\tablemidrule
OPT Teacher & 0.125B & 43.4 & 19.3 & 52.4 & 54.2 & 62.4 & 31.9 & 20.3 & 23.5 & 38.4 & 125.6 & NA & 126.0 \\
KD & 0.125B & 37.1 & 19.1 & 51.2 & 47.8 & 53.8 & 27.6 & 18.4 & 22.2 & 34.7 & 27.7 & 0.17 & 9.3 \\
SpikeGPT & 0.046B & 32.3 & 16.2 & 50.2 & 45.7 & 54.6 & 25.3 & 15.7 & 20.6 & 32.6 & \textbf{3.7} & 0.17 & \textbf{3.3} \\
SpikeGPT & 0.216B & 35.2 & 17.7 & 50.7 & 47.3 & 55.1 & 27.6 & 17.3 & 23.1 & 34.3 & 18.3 & 0.17 & 16.5 \\
\rowcolor{tableours}
\method{} & 0.125B & \textbf{39.1} & \textbf{19.2} & \textbf{51.9} & \textbf{48.5} & \textbf{55.3} & \textbf{28.1} & \textbf{19.0} & \textbf{22.6} & $\mathbf{35.5}$ & 27.5 & $0.17$ & 9.2 \\
\midrule
OPT Teacher & 0.35B & 47.5 & 22.2 & 55.3 & 57.2 & 66.1 & 40.7 & 25.7 & 26.6 & 42.7 & 360.8 & N/A & 197.6 \\
KD & 0.35B & 42.1 & 20.7 & 50.8 & 55.3 & 60.5 & 32.8 & \textbf{21.8} & 22.3 & 38.3 & 85.7 & 0.18 & 16.9 \\
\rowcolor{tableours}
\method{} & 0.35B & \textbf{44.2} & \textbf{20.9} & \textbf{56.1} & \textbf{55.7} & \textbf{62.2} & \textbf{35.0} & \textbf{21.8} & \textbf{24.2} & $\mathbf{40.0}$ & \textbf{83.7} & $0.17$ & \textbf{16.7} \\
\midrule
OPT Teacher & 1.3B & 57.4 & 30.6 & 60.4 & 60.5 & 71.7 & 52.8 & 33.2 & 30.7 & 49.7 & 1237.1 & N/A & 632.2 \\
KD & 1.3B & 44.5 & 23.6 & 55.2 & 56.9 & 62.8 & 40.8 & 25.2 & 22.3 & 41.4 & 130.7 & 0.18 & 67.3 \\
\rowcolor{tableours}
\method{} & 1.3B & \textbf{48.9} & \textbf{25.6} & \textbf{58.3} & \textbf{57.4} & \textbf{65.5} & \textbf{46.3} & \textbf{27.6} & \textbf{24.6} & $\mathbf{44.3}$ & \textbf{128.3} & 0.19 & \textbf{66.8} \\
\tablebottomrule
\end{tabular}
}
\caption{Zero-shot accuracy and analytical efficiency. For \method{}, all task accuracies, Avg., and Rate are three-seed means; the corresponding sample standard deviations for Avg. and Rate are reported in the text. Avg. is unweighted. OPs and Energy are reported in G and mJ. Bold marks the best matched-scale SNN accuracy, including ties.}
\label{tab:main}
\end{table*}
}

\subsection{Layerwise Spike Regularization}

Matched-prefix policy anchoring constrains output behavior, while layerwise spike activity remains an independent degree of freedom. Layers can compensate to preserve logits while some become silent or overly active. We therefore regularize selected layers rather than a network-wide spike count. For $B$ trajectories, layer $l$ fires at rate
\begin{equation}
r_l=\frac{1}{BKT_sN_l}\sum_{b,t,\tau,i}s_{b,l,i}^{t,\tau}.
\label{eq:firing_rate}
\end{equation}
Let $r_l^{\mathrm{ref}}$ be the rate of the frozen reference SNN on the same prefixes. We use the fixed layer set $\mathcal S_{\mathrm{spk}}=\{3,6,9,12\}$ at all three model scales and penalize deviations both from a viable interval and from the frozen reference SNN rate:
\begin{equation}
\begin{aligned}
\loss_{\mathrm{spk}}=\frac{1}{|\mathcal S_{\mathrm{spk}}|}
\sum_{l\in\mathcal S_{\mathrm{spk}}}\big(&
\pospart{r_{\min}-r_l}^{2}
+\pospart{r_l-r_{\max}}^{2}\\
&+\rho(r_l-r_l^{\mathrm{ref}})^2\big).
\end{aligned}
\label{eq:spike}
\end{equation}
The interval discourages silent or saturated selected layers, even if the reference approaches a boundary. The reference term penalizes deviations from the firing rates of the frozen reference SNN on self-generated prefixes. Per-layer penalties prevent increases and decreases from cancelling globally. We set $\rho=1$ and use the same frozen reference SNN used for matched-prefix policy anchoring.

\subsection{Joint Training Objective}

Because $\loss_{\mathrm{OPD}}$ changes the active policy, it also changes the prefixes and internal rates observed at the next update. We therefore evaluate all objectives on the same fresh prefix batch:
\begin{equation}
\loss_{\mathrm{total}}
=
\loss_{\mathrm{OPD}}
+\beta_{\mathrm{ref}}\loss_{\mathrm{ref}}
+\lambda_{\mathrm{spk}}\loss_{\mathrm{spk}}.
\label{eq:total}
\end{equation}
$\loss_{\mathrm{OPD}}$ corrects the output policy, $\loss_{\mathrm{ref}}$ limits departure from the reference policy, and $\loss_{\mathrm{spk}}$ constrains selected-layer firing dynamics. Only the active SNN parameters $\theta$ are updated; the ANN teacher and reference SNN remain frozen. We keep $\beta_{\mathrm{ref}}$ and $\lambda_{\mathrm{spk}}$ fixed throughout \method{} adaptation. Appendix~B gives the complete update sequence, and Appendix~D derives the gradients of both stability terms.

After training, the teacher and reference SNN are discarded. Deployment uses only the updated SNN, without an auxiliary model, reranker, or second decoding pass.

\mainresultstable

\section{Experiments}

\subsection{Experimental Setup}

\paragraph{Benchmark Datasets.}
We evaluate retained language ability on eight zero-shot benchmarks: ARC-Easy/Challenge (ARC-e/ARC-c) \citep{clark2018think}, WinoGrande (WG) \citep{sakaguchi2020winogrande}, BoolQ (BQ) \citep{clark2019boolq}, PIQA \citep{bisk2020piqa}, HellaSwag (HS) \citep{zellers2019hellaswag}, OpenBookQA (OBQA) \citep{mihaylov2018can}, and HeadQA (HQA) \citep{vilares2019head}. \method{} sees no examples from these tasks. We report each accuracy and their unweighted mean.

\paragraph{Compared Baselines.}
All controlled 0.125B baselines share one KD checkpoint and update budget. We compare continued \spad{}, hard-label SFT, Vanilla OPD, and the on-policy EOPD, Lightning OPD, and OPRD variants \citep{jin2026entropy,wu2026lightning,yang2026oprd}. The scale study reports OPT teachers \citep{zhang2022opt}, KD and \method{}-adapted \bispikclm{} models at 0.125B, 0.35B, and 1.3B, plus published SpikeGPT checkpoints \citep{zhu2023spikegpt}.

\paragraph{Evaluation Metrics.}
We report accuracy on each of the eight zero-shot benchmarks and their unweighted average to measure retained language capability. To assess generation stability under the same decoding setting, we use adjacent repetition (Adj.), the fraction of unique four-token sequences (Dist.-4), and the maximum identical-token run length (Run) \citep{holtzman2019curious,welleck2019neural}. Lower Adj. and Run indicate fewer local repetitions and severe degeneration, whereas higher Dist.-4 indicates greater sequence diversity; jointly, they distinguish isolated repetition from broader rollout collapse. Fire Rate follows Eq.~\ref{eq:firing_rate} and measures the sparse activation level of the SNN. Operations and energy are analytical estimates computed using the matched \bispikclm{} protocol \citep{guo2026bispikclm}, rather than hardware latency or power measurements; Appendix~E specifies the calculation boundary, formulas, and operation-energy constants.

\paragraph{Implementation Details.}
\method{} starts from a 4,000-update \spad{} checkpoint and 480-token FineWeb prompts \citep{penedo2024fineweb}, generating $K=32$ tokens at temperature 1.0 with $T_s=4$. We train for 500 updates using $(1,\beta_{\mathrm{ref}},\lambda_{\mathrm{spk}})=(1,0.75,0.3)$ and $\mathcal S_{\mathrm{spk}}=\{3,6,9,12\}$. Controlled experiments use 0.125B models, and the scale study adds 0.35B and 1.3B. Scale results use three seeds, the stability test uses 10 matched seeds, and paired-prefix diagnostics use $N=256$ fixed pairs. Remaining implementation and reproducibility details are in the appendix.

\begin{table*}[!t]
\centering
{
\small
\rowcolors{3}{tablebody}{white}
\begin{tabular}{lccccccccc|ccc}
\tabletoprule

\rowcolor{tableheader}
\rule{0pt}{2.8ex}
& \multicolumn{9}{c|}{\textbf{Zero-shot Accuracy (\%)}}
& \multicolumn{3}{c}{\textbf{Rollout Statistics}} \\

\noalign{%
  \global\aboverulesep=0pt
  \global\belowrulesep=0pt
}
\cmidrule[0.1pt](lr){2-10}
\cmidrule[0.1pt](lr){11-13}
\noalign{%
  \global\aboverulesep=.4ex
  \global\belowrulesep=.65ex
}

\rowcolor{tableheader}
\rule{0pt}{2.8ex}
Method & ARC-e & ARC-c & WG & BQ & PIQA & HS & OBQA & HQA & Avg.
& Adj.\,\smash{$\downarrow$}
& Dist.-4\,\smash{$\uparrow$}
& Run\,\smash{$\downarrow$} \\

\tablemidrule
KD & 37.1 & 19.1 & 51.2 & 47.8 & 53.8 & 27.6 & 18.4 & 22.2 & 34.7 & 0.16 & \textbf{99.83} & 1.0433 \\
Continued \spad{} & 37.3 & 19.1 & 49.6 & 48.0 & 54.2 & 27.3 & 18.4 & 22.0 & 34.5 & 0.16 & 99.77 & 1.0457 \\
SFT & 36.8 & 16.9 & 50.6 & 38.4 & 49.1 & 25.9 & 20.0 & 18.8 & 32.1 & \textbf{0.09} & 7.28 & \textbf{1.0275} \\
Vanilla OPD & 36.0 & 17.0 & 49.0 & 38.4 & 49.9 & 26.1 & \textbf{22.0} & 20.3 & 32.3 & 39.09 & 83.67 & 6.9041 \\
EOPD & 37.3 & 17.5 & 51.3 & 46.2 & 54.1 & 27.3 & 21.2 & 21.7 & 34.6 & 24.77 & 87.21 & 5.8413 \\
Lightning OPD & 37.1 & \textbf{19.2} & 50.1 & 45.8 & 51.5 & 26.8 & 21.1 & 21.7 & 34.2 & 0.13 & 99.23 & 1.0513 \\
OPRD & 36.8 & 18.7 & 48.9 & 46.2 & 51.3 & 26.7 & 21.6 & 22.0 & 34.0 & 0.16 & 99.72 & 1.0836 \\
\rowcolor{tableours}
\method{} & \textbf{39.1} & \textbf{19.2} & \textbf{51.9} & \textbf{48.5} & \textbf{55.3} & \textbf{28.1} & 19.0 & \textbf{22.6} & \textbf{35.5} & 0.11 & 99.19 & 1.0477 \\
\tablebottomrule
\end{tabular}
}
\caption{Controlled 0.125B comparison from one checkpoint and 500 updates. Adj., Dist.-4, and Run denote adjacent repetition (\%), unique four-token fraction (\%), and maximum identical-token run length. Bold marks the best accuracy, including ties.}
\label{tab:controlled}
\end{table*}

\subsection{Performance Comparison with KD Baselines}

We test whether on-policy adaptation scales without sacrificing sparse execution by comparing \method{} with matched KD checkpoints at 0.125B, 0.35B, and 1.3B. For \method{}, Table~\ref{tab:main} reports three-run means for every task accuracy, the unweighted eight-task average, and firing rate. The eight-task averages are $35.5\pm0.02\%$, $40.0\pm0.19\%$, and $44.3\pm0.12\%$, and the firing rates are $0.17\pm0.006$, $0.17\pm0.000$, and $0.19\pm0.010$, respectively, where each uncertainty is the sample standard deviation across the three runs.

The benefit is broad and becomes more pronounced over the tested scale range: \method{} improves the KD average by 0.8, 1.7, and 2.9 points. The three-run task means improve on every reported task at 0.125B and 1.3B; at 0.35B, seven improve and OpenBookQA is unchanged. These improvements require neither downstream-task training nor additional analytical compute; operation and energy estimates remain slightly below KD, while firing rates remain comparable. This pattern suggests that correcting prefix-source mismatch recovers language capability without compromising the sparse inference profile.

\subsection{Comparison with Post-Training Methods}

We compare post-training methods with different supervision and data sources under the same 0.125B KD checkpoint and 500-update budget, including continued offline training, hard-label SFT, and existing on-policy distillation methods. Table~\ref{tab:controlled} evaluates both task accuracy and rollout statistics to determine whether post-training gains coexist with stable generation.

The controlled comparison reveals that effective post-training requires more than additional updates or exposure to self-generated prefixes. Existing methods either preserve rollout statistics without improving accuracy or update on-policy and suffer generation degeneration. \method{} is the only method that raises the KD average from 34.7\% to 35.5\% while keeping rollout statistics close to the initial checkpoint. This result indicates that teacher correction must be coupled with explicit control of policy movement and internal spiking dynamics to break the observed accuracy--stability trade-off.

\subsection{Ablation Study}

We evaluate the three components of \method{} by individually removing teacher-guided policy correction, matched-prefix policy anchoring, or layerwise spike regularization from Eq.~\ref{eq:total}. All variants use the same 0.125B checkpoint, self-generated prefixes, optimization schedule, and update budget. Figure~\ref{fig:ablation} reports their accuracy across eight zero-shot benchmarks and the average.

\begin{figure}[!htbp]
    \centering
    \includegraphics[width=\columnwidth]{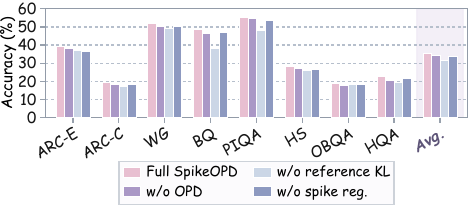}
    \caption{Component ablation at 0.125B. ``w/o reference KL'' removes matched-prefix policy anchoring. The other labels remove the named term in Eq.~\ref{eq:total}.}
    \label{fig:ablation}
\end{figure}

The full configuration achieves the highest average accuracy of 35.5\%, while removing any component degrades performance. \textbf{Obs.~1:} Teacher-guided correction provides the optimization direction. Without $\mathcal{L}_{\mathrm{OPD}}$, the preservation objectives can restrict policy changes but cannot transfer the ANN teacher's distribution on self-generated prefixes, leaving prefix-source mismatch insufficiently corrected. \textbf{Obs.~2:} Matched-prefix policy anchoring is the primary stability constraint. Removing $\mathcal{L}_{\mathrm{ref}}$ produces the largest decline, from 35.5\% to 31.7\%, showing that teacher supervision alone cannot ensure stable adaptation. Since policy updates also alter subsequent prefixes, unconstrained movement can amplify early deviations. The frozen reference limits this feedback on the same visited prefixes without blocking teacher correction. \textbf{Obs.~3:} Spike regularization provides a complementary internal constraint. Its removal reduces accuracy because output alignment cannot fully preserve layerwise spiking behavior, which may drift even when similar logits are maintained. This result is consistent with Figure~\ref{fig:dynamics}(A). Overall, the three objectives jointly control the correction direction, policy movement, and internal spiking dynamics.

\subsection{Initialization Robustness}

According to Figure~\ref{fig:initialization}, we test whether \method{} depends on a mature offline checkpoint by adapting models after 1,000, 2,000, and 4,000 \spad{} updates and comparing each result with its matched initialization.

\begin{figure}[!htbp]
    \centering
    \includegraphics[width=\columnwidth]{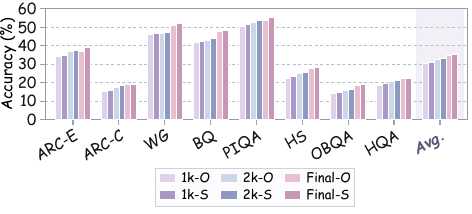}
    \caption{Initialization robustness across 1k, 2k, and final 4k offline checkpoints. O and S denote the matched KD and \method{} results. All three pairs improve.}
    \label{fig:initialization}
\end{figure}

\method{} improves all three matched checkpoints by 0.6--0.8 points, although initialization maturity still determines absolute performance. The consistent relative gain shows that \method{} complements rather than replaces spike-aware offline alignment: offline training establishes language capability, whereas on-policy adaptation addresses the prefix-source mismatch that appears when the model conditions on its own generations.

\subsection{Sensitivity Analysis}

We study two hyperparameters that control adaptation strength: rollout length $K$ determines the extent of self-generated context, while the layer interval $\Delta L$ determines the coverage of internal spike constraints. Tables~\ref{tab:sensitivity_k} and~\ref{tab:sensitivity_layers} vary them at 0.125B under a fixed update budget and report both accuracy and rollout quality.

\begin{table}[!htbp]
\centering
{
\small
\rowcolors{2}{tablebody}{white}
\begin{tabular}{lccccc}
\tabletoprule
\rowcolor{tableheader}
\tablesingleheadpad $K$ & Upd. & Gen./upd. & Avg. $\uparrow$ & Adj. $\downarrow$ & Dist.-4 $\uparrow$ \\
\tablemidrule
8 & 500 & $8B$ & 34.9 & 0.2 & 99.0 \\
16 & 500 & $16B$ & 35.3 & \textbf{0.1} & \textbf{99.3} \\
\rowcolor{tableours}
\textbf{32} & 500 & $32B$ & $\mathbf{35.5}$ & \textbf{0.1} & 99.2 \\
64 & 500 & $64B$ & 35.2 & 0.2 & 99.1 \\
128 & 500 & $128B$ & 34.8 & 0.3 & 98.8 \\
\tablebottomrule
\end{tabular}
}
\caption{Rollout-length sensitivity at 0.125B. Metrics are percentages. The shaded default uses three seeds and other rows use one fixed seed.}
\label{tab:sensitivity_k}
\end{table}

\begin{table}[!htbp]
\centering
{
\small
\rowcolors{2}{tablebody}{white}
\begin{tabular}{lcccc}
\tabletoprule
\rowcolor{tableheader}
\tablesingleheadpad Layers & Avg. $\uparrow$ & Adj. $\downarrow$ & Dist.-4 $\uparrow$ & Rate \\
\tablemidrule
$\Delta L=1$ & 35.0 & 0.2 & 99.0 & 0.176 \\
$\Delta L=2$ & 35.3 & \textbf{0.1} & \textbf{99.3} & 0.171 \\
\rowcolor{tableours}
\textbf{$\Delta L=3$} & $\mathbf{35.5}$ & \textbf{0.1} & 99.2 & $0.173$ \\
$\Delta L=4$ & 35.2 & \textbf{0.1} & 99.1 & 0.168 \\
$\Delta L=6$ & 34.9 & 0.2 & 99.0 & 0.164 \\
$\Delta L=12$ & 34.6 & 0.3 & 98.8 & 0.158 \\
\tablebottomrule
\end{tabular}
}
\caption{Layer-selection sensitivity at 0.125B. None removes spike regularization. Metrics are percentages except Rate. The shaded default uses three seeds and other rows use one.}
\label{tab:sensitivity_layers}
\end{table}

The shared finding is that stable on-policy adaptation benefits from bounded self-generated context and selective internal constraints rather than maximal coverage. Short rollouts provide insufficient on-policy exposure, whereas excessively long rollouts may accumulate feedback from the evolving policy; similarly, dense spike constraints may suppress useful adaptation, while sparse constraints leave internal drift undercontrolled. The intermediate optima at $K=32$ and $\Delta L=3$ in Tables~\ref{tab:sensitivity_k} and~\ref{tab:sensitivity_layers} support this interpretation.

\section{Conclusion}
We propose SpikeOPD, which employs full KL teacher correction, prefix-matching strategy anchoring, and layer-wise spike regularization on self-generated prefixes. Across three different scales, SpikeOPD improves knowledge distillation performance by 0.8, 1.7, and 2.9 percentage points, respectively, without increasing analytical sparse computation overhead or requiring auxiliary models at deployment time. This study offers a novel technical pathway for autoregressive SNN distillation.

\bibliography{references}

\clearpage
\appendix
\setcounter{secnumdepth}{1}
\setcounter{section}{0}
\setcounter{figure}{0}
\setcounter{table}{0}
\setcounter{equation}{0}

\section{Overview and Notation}
\label{app:overview}

This supplement provides the Stage~2 training protocol, the properties of full-vocabulary on-policy distillation, the two stability terms, the shared-prefix diagnostic protocol, and the analytical-efficiency totals reported in the main paper. It then gives an implementation-level record of the offline \spad{} initialization, including cross-domain attention and feature alignment, the exact Stage~1 optimization settings, and a reset-aware finite-horizon surrogate-gradient analysis. It uses the same symbols and reporting boundaries as the main paper.

Let $x\sim\mathcal{D}_{\mathrm{prompt}}$ be a corpus prompt and let $y_{1:K}$ be a continuation sampled from an active spiking student. The state at response position $t$ is $s_t=(x,y_{<t})$. The frozen ANN teacher, active SNN, and frozen offline SNN reference define distributions over the complete vocabulary $\V$:
\begin{equation}
\begin{aligned}
p_T(v\mid s_t)&=\operatorname{softmax}(z_T(s_t))_v,\\
q_\theta(v\mid s_t)&=\operatorname{softmax}(z_\theta(s_t))_v,\\
q_{\mathrm{ref}}(v\mid s_t)&=\operatorname{softmax}(z_{\mathrm{ref}}(s_t))_v.
\end{aligned}
\label{eq:app_distributions}
\end{equation}
The reference parameters equal the Stage~1 checkpoint and remain frozen. Rollout token identities are detached from the computation graph. Consequently, the sampled sequence chooses the states on which supervision is applied, while gradients at an update are computed with the visited prefixes held fixed.

\section{Detailed Experimental Protocol}
\label{app:protocol}

\subsection{Models, data, and two-stage training}

We study matched OPT teachers and causal \bispikclm{} students at 125M, 0.35B, and 1.3B parameters \citep{zhang2022opt,guo2026bispikclm}. Stage~1 uses the FineWeb-Edu subset specified in Section~\ref{app:stage1_optimization}; Stage~2 draws prompts from FineWeb \citep{penedo2024fineweb}. Stage~1 performs offline spike-aware ANN-to-SNN distillation for 4,000 updates. Stage~2 starts from that checkpoint, copies it once to form the frozen reference SNN, and updates only the active SNN.

At every Stage~2 update, we sample a batch of 480-token FineWeb prompts. The active SNN generates $K=32$ new tokens with sampling temperature $1.0$. The frozen teacher and frozen reference then score every visited prefix $s_t$ over the full vocabulary. The default accuracy and controlled-comparison runs use 500 updates, batch size 16, and learning rate $1.5\times10^{-6}$. The separate stability stress test follows the horizons reported in the main paper: 500 updates for Vanilla and Clipped OPD, and 1,000 updates for Small-LR OPD and \method{}. No example from any downstream benchmark is used for gradient-based optimization.

\begin{table}[t]
\centering
\small
\setlength{\tabcolsep}{4pt}
\begin{tabular}{ll}
\toprule
Setting & Value \\
\midrule
Stage~1 checkpoint & 4,000 updates \\
Prompt length & 480 tokens \\
Rollout length $K$ & 32 tokens \\
Rollout temperature & 1.0 \\
Simulation steps $T_s$ & 4 \\
Diagnostic prompt count & 256 \\
Default Stage~2 horizon & 500 updates \\
Stability-stress horizon & 500 or 1,000 updates \\
Global trajectory batch $B$ & 16 \\
Learning rate & $1.5\times10^{-6}$ \\
Teacher correction & Full-vocabulary forward KL \\
Loss weights & $1.00:0.75:0.30$ \\
Spike-rate interval & $[0.01,0.58]$ \\
Reference-rate coefficient $\rho$ & 1 \\
Default training layers & $\{3,6,9,12\}$ \\
Trajectory diagnostics & All $L$ spiking layers (125M only) \\
Trainable network & Active SNN only \\
\bottomrule
\end{tabular}
\caption{Stage~2 reproducibility record. The objective weights are fixed throughout each run, including the 1,000-update stability runs.}
\label{tab:app_reproducibility}
\end{table}

\subsection{Fixed Stage~2 objective}

For a batch of $B$ sampled trajectories, the complete teacher correction is
\begin{equation}
\loss_{\mathrm{OPD}}
=\frac{1}{BK}\sum_{b=1}^{B}\sum_{t=1}^{K}
\KL\!\left(p_T(\cdot\mid s_{b,t})\,\|\,
q_\theta(\cdot\mid s_{b,t})\right).
\label{eq:app_opd}
\end{equation}
The matched-prefix reference penalty is
\begin{equation}
\loss_{\mathrm{ref}}
=\frac{1}{BK}\sum_{b=1}^{B}\sum_{t=1}^{K}
\KL\!\left(q_\theta(\cdot\mid s_{b,t})\,\|\,
q_{\mathrm{ref}}(\cdot\mid s_{b,t})\right).
\label{eq:app_ref}
\end{equation}
For a batch of $B$ trajectories and each monitored layer $l\in\mathcal S_{\mathrm{spk}}$, let $s_{b,l,i}^{t,\tau}\in\{0,1\}$ be the spike of neuron $i$ in batch item $b$ at response position $t$ and simulation step $\tau$. Its batch-averaged rate is
\begin{equation}
r_l=\frac{1}{B K T_s N_l}\sum_{b=1}^{B}\sum_{t=1}^{K}\sum_{\tau=1}^{T_s}\sum_{i=1}^{N_l}s_{b,l,i}^{t,\tau},
\label{eq:app_rate}
\end{equation}
Let $r_l^{\mathrm{ref}}$ be the rate produced by the frozen reference on the same prefixes. The activity penalty is
\begin{equation}
\begin{aligned}
\loss_{\mathrm{spk}}
&=\frac{1}{|\mathcal S_{\mathrm{spk}}|}
\sum_{l\in\mathcal S_{\mathrm{spk}}}
\bigl(\pospart{r_{\min}-r_l}^{2}
+\pospart{r_l-r_{\max}}^{2}\\
&\hspace{8.2em}+\rho(r_l-r_l^{\mathrm{ref}})^2\bigr),
\end{aligned}
\label{eq:app_spike}
\end{equation}
where $\rho=1$ throughout. The default/full method and all non-layer-sensitivity experiments use $\mathcal S_{\mathrm{spk}}=\{3,6,9,12\}$ at every model scale. The layer-sensitivity study instead uses the set $\{\Delta L,2\Delta L,\ldots,12\}$ specified by each tested interval. The all-layer quantities used only for trajectory diagnosis in Section~\ref{app:protocol} are not the training regularizer. The outer Stage~2 objective is fixed for the entire run:
\begin{equation}
\boxed{
\loss_{\mathrm{total}}
=1.00\,\loss_{\mathrm{OPD}}
+0.75\,\loss_{\mathrm{ref}}
+0.30\,\loss_{\mathrm{spk}}.}
\label{eq:app_total}
\end{equation}
There is no weight schedule and \method{} does not replay a Stage~1 loss during Stage~2. Continued \spad{} is the sole diagnostic control that intentionally reuses the Stage~1 objective for 500 additional offline updates. Thus, the ratio $1:0.75:0.3$ always refers to full-KL teacher correction, matched-prefix policy anchoring, and layerwise spike regularization, respectively. The latter two terms form a coupled two-level stability module: $\loss_{\mathrm{ref}}$ constrains policy drift on matched prefixes, while $\loss_{\mathrm{spk}}$ constrains spike activity. Full \method{} always uses both terms together; single-term removals in the ablation are diagnostic only.

\subsection{Update sequence}

Each Stage~2 update executes the following sequence:
\begin{enumerate}
    \item Sample a batch of 480-token FineWeb prompts.
    \item Generate 32 tokens from the active SNN at temperature $1.0$ and detach the sampled token identities.
    \item Run the frozen ANN teacher, active SNN, and frozen reference SNN on the same visited prefixes.
    \item Compute the complete vocabulary sums in Eqs.~\ref{eq:app_opd} and~\ref{eq:app_ref}; no vocabulary truncation or token-level Monte Carlo estimator is used.
    \item Accumulate active and reference firing rates for the layer set $l\in\mathcal S_{\mathrm{spk}}$ of the current configuration and evaluate Eq.~\ref{eq:app_spike}; the default set is $\{3,6,9,12\}$.
    \item Form Eq.~\ref{eq:app_total}, backpropagate through the active SNN with surrogate gradients, and update only $\theta$.
\end{enumerate}
The teacher and reference are discarded after training. Inference therefore uses the same student architecture and introduces no auxiliary network or new module.

\subsection{Controlled baselines and evaluation}

All controlled 125M post-training comparisons start from the same 4,000-update checkpoint, use $T_s=4$, and receive the same 500-update budget. \emph{Offline \spad{}} evaluates the checkpoint without Stage~2. \emph{Continued \spad{}} applies the original offline objective for 500 additional FineWeb updates. \emph{SFT} performs standard causal-language-model fine-tuning on the same FineWeb sequences: at each corpus prefix, the ground-truth next token from that sequence is the hard target in a cross-entropy loss. It does not use tokens sampled from the active rollout as labels. Vanilla OPD applies full forward KL alone and removes the complete coupled two-level stability module; it omits both $\loss_{\mathrm{ref}}$ and $\loss_{\mathrm{spk}}$. \emph{\method{}} uses Eq.~\ref{eq:app_total}. SpikeGPT is reported only as external context and is not a training-matched control \citep{zhu2023spikegpt}.

The recent-method controls preserve the same prompt length, continuation length, temperature, and update budget. \emph{Lightning OPD} \citep{wu2026lightning} first generates the complete fixed trajectory buffer from the 4,000-update checkpoint and caches frozen-teacher full-vocabulary scores on those prefixes. Its 500 updates sample only this buffer, without refreshing trajectories as the SNN changes, and include the same selected-layer $\loss_{\mathrm{spk}}$ used by \method{}. \emph{OPRD} \citep{yang2026oprd} uses active-student rollouts and temporally averages SNN hidden states across all four simulation steps. A learned linear map aligns each of the $L$ student layers to its matched teacher layer with MSE; the representation coefficient is calibrated on the first five minibatches so its gradient norm is of the same order as the common full-KL calibration signal, then held fixed.

\emph{EOPD} \citep{jin2026entropy} uses the frozen teacher's token entropy to switch between reverse KL at low entropy and forward KL at high entropy. The high-entropy branch uses the teacher-probability Top-50 approximation, and the threshold is fixed to the 80th percentile of one teacher-scored calibration minibatch. For a fair SNN comparison, EOPD retains the same selected-layer $\loss_{\mathrm{spk}}$ as \method{} but does not use the matched-prefix reference-policy term. These controls isolate entropy-aware divergence switching, fixed replay, and representation supervision under a common SNN budget.

Zero-shot accuracy is evaluated on ARC-Easy and ARC-Challenge \citep{clark2018think}, WinoGrande \citep{sakaguchi2020winogrande}, BoolQ \citep{clark2019boolq}, PIQA \citep{bisk2020piqa}, HellaSwag \citep{zellers2019hellaswag}, OpenBookQA \citep{mihaylov2018can}, and HeadQA \citep{vilares2019head}. The scale study uses three independent runs per scale, and the main table reports three-run means for every \method{} task accuracy, the unweighted eight-task average, and firing rate. The sample standard deviations of the eight-task average are $0.02$, $0.19$, and $0.12$ percentage points at 125M, 0.35B, and 1.3B, respectively; the corresponding firing-rate standard deviations are $0.006$, $0.000$, and $0.010$. The collapse stress test uses the same 10 matched seeds for Vanilla OPD, Small-LR OPD, Clipped OPD, and \method{}. Unless a table explicitly reports a multi-seed aggregate, other controlled ablations and sensitivity settings are single runs. The downstream evaluation examples are not mixed into either training stage.

\subsection{Diagnostic definitions}

For a generated continuation $y_{1:K}$, adjacent repetition is
\begin{equation}
R_{\mathrm{adj}}(y)=\frac{1}{K-1}\sum_{t=2}^{K}\mathbf{1}[y_t=y_{t-1}].
\end{equation}
The maximum run is the length of the longest contiguous sequence of an identical token. Distinct-4 is the number of unique contiguous 4-grams divided by the total number $K-3$. The repeated-4-gram fraction used in the training trace is the fraction of positions whose current 4-gram already appeared earlier in the same continuation. These complementary statistics distinguish local token loops from broader phrase reuse; lower repetition and shorter runs are better, whereas higher distinct-4 is better \citep{holtzman2019curious,welleck2019neural}.

At rollout depth $t$, teacher--student KL is evaluated before the average over response positions in Eq.~\ref{eq:app_opd}. For checkpoint $c$, prompt $n$, and source $a\in\{\mathrm{self},\mathrm{corp}\}$, let
\begin{equation}
\begin{aligned}
\kappa_{c,n,t}^{a}
&=\KL\!\left(p_T(\cdot\mid s_{n,t}^{a})\,\|\,
q_c(\cdot\mid s_{n,t}^{a})\right),\\
G_c(t)
&=\frac{1}{N}\sum_{n=1}^{N}
\left(\kappa_{c,n,t}^{\mathrm{self}}
-\kappa_{c,n,t}^{\mathrm{corp}}\right).
\end{aligned}
\label{eq:app_diagnostic_kl}
\end{equation}
Absolute curves average $\kappa_{c,n,t}^{a}$ over the 256 prompts, while $G_c(t)$ is the paired source gap. Internal-dynamics diagnostics use the same self-generated prefix and length-matched corpus-prefix pair.

We formalize all three internal diagnostics as follows. Let $c\in\{c_{\mathrm{off}},c_{\mathrm{van}},c_{\mathrm{meth}}\}$ index the fixed 4,000-update offline checkpoint, the Vanilla OPD checkpoint after 500 Stage~2 updates, and the \method{} checkpoint after 500 Stage~2 updates, respectively. Let $a\in\{\mathrm{self},\mathrm{corp}\}$ index a stored self-generated prefix and its length-matched corpus-prefix control from the same prompt and position. For prompt $n\in\{1,\ldots,N\}$, response position $t\in\{1,\ldots,K\}$, layer $l\in\{1,\ldots,L\}$ of width $d_l$ and $N_l$ spiking neurons, and simulation step $\tau\in\{1,\ldots,T_s\}$, write $s_{n,t,l,i,\tau}^{c,a}\in\{0,1\}$ for a spike, $H_{n,t,l,\tau}^{c,a}\in\mathbb{R}^{d_l}$ for the hidden vector, and $U_{n,t,l,\tau}^{c,a}\in\mathbb{R}^{d_l}$ for the corresponding membrane-potential vector. Here $N=256$, $K=32$, and $T_s=4$.

For a single prompt, position, layer, checkpoint, and prefix type, define the simulation-time and neuron-averaged spike rate
\begin{equation}
r_{n,t,l}^{c,a}
=\frac{1}{T_sN_l}\sum_{\tau=1}^{T_s}\sum_{i=1}^{N_l}s_{n,t,l,i,\tau}^{c,a}.
\label{eq:diag_rate}
\end{equation}
The checkpoint-specific spike-rate drift at rollout depth $t$ is
\begin{equation}
D_{\mathrm{spk}}(c,t)
=\frac{1}{NL}\sum_{n=1}^{N}\sum_{l=1}^{L}
\left|r_{n,t,l}^{c,\mathrm{self}}-r_{n,t,l}^{c,\mathrm{corp}}\right|.
\label{eq:diag_spike_drift}
\end{equation}
Figure~2A in the main paper plots $10^4D_{\mathrm{spk}}(c,t)$ for Offline-only, Vanilla OPD, and \method{}; the multiplier is a display scale, not a physical unit. The direct Offline-only--\method{} comparison is
\begin{equation}
\Delta_{\mathrm{spk}}(t)
=D_{\mathrm{spk}}(c_{\mathrm{off}},t)
-D_{\mathrm{spk}}(c_{\mathrm{meth}},t),
\label{eq:diag_spike_improvement}
\end{equation}
where $\Delta_{\mathrm{spk}}(t)>0$ means that \method{} has a smaller spike-rate gap than Offline-only on the same stored prefixes and controls. Vanilla OPD is evaluated under the same protocol as an intermediate diagnostic control, while $\Delta_{\mathrm{spk}}(t)$ retains the direct Offline-only--\method{} definition used for the reported improvement.

The hidden vector is first aggregated over simulation time:
\begin{equation}
\bar H_{n,t,l}^{c,a}
=\frac{1}{T_s}\sum_{\tau=1}^{T_s}H_{n,t,l,\tau}^{c,a}.
\label{eq:app_diag_hidden_aggregate}
\end{equation}
For every coordinate $j$, prompt, position, and layer, define the absolute hidden discrepancy and its layer mean by
\begin{equation}
\begin{aligned}
\delta_{n,t,l,j}^{H,c}
&=\left|\bar H_{n,t,l,j}^{c,\mathrm{self}}
-\bar H_{n,t,l,j}^{c,\mathrm{corp}}\right|,\\
g_{n,t,l}^{H,c}
&=\frac{1}{d_l}\sum_{j=1}^{d_l}\delta_{n,t,l,j}^{H,c},\\
G_H(c)
&=\frac{1}{N K L}\sum_{n=1}^{N}\sum_{t=1}^{K}
\sum_{l=1}^{L}g_{n,t,l}^{H,c}.
\end{aligned}
\label{eq:app_diag_hidden_gap}
\end{equation}
Thus $g_{n,t,l}^{H,c}$ is the hidden-state gap for one prompt--position--layer item, and $G_H(c)$ is the reported checkpoint-level gap. Membrane drift retains the simulation-step axis instead of averaging the trace before taking the distance:
\begin{equation}
\begin{aligned}
D_U(c)
&=\frac{1}{N K L T_s}
\sum_{n=1}^{N}\sum_{t=1}^{K}\sum_{l=1}^{L}
\sum_{\tau=1}^{T_s}\\[-2pt]
&\quad\frac{1}{d_l}
\left\|U_{n,t,l,\tau}^{c,\mathrm{self}}
-U_{n,t,l,\tau}^{c,\mathrm{corp}}\right\|_1.
\end{aligned}
\label{eq:app_diag_membrane_drift}
\end{equation}
Only the hidden state is averaged over simulation time before taking the distance. Equation~\ref{eq:app_diag_membrane_drift} instead computes a membrane distance at each of the four simulation steps and then averages over $\tau$. In both diagnostics, $1/d_l$ produces a coordinate mean within a layer, and the outer $1/L$ gives every layer equal weight rather than weighting wider layers more heavily. The reported values are $G_H(c_{\mathrm{off}})=0.102$, $G_H(c_{\mathrm{van}})=0.100$, and $G_H(c_{\mathrm{meth}})=0.095$ for hidden states, and $D_U(c_{\mathrm{off}})=3.46$, $D_U(c_{\mathrm{van}})=3.30$, and $D_U(c_{\mathrm{meth}})=2.67$ for membrane potentials.

Both quantities use native model-internal units. The membrane coordinates are not physical volts, are not divided by $V_{\mathrm{th}}$, are not variance-normalized, and are not $z$-scored; the hidden coordinates likewise receive no variance normalization or $z$-scoring beyond the explicit temporal and coordinate means above. Consequently, these raw diagnostics support paired checkpoint comparisons only when architecture, parameterization, activation scale, prefix bank, and aggregation are held fixed. They must not be compared directly across model scales or differently parameterized SNNs without an additional shared normalization. These quantities are diagnostics only and are not direct optimization targets except for the firing-rate statistics in Eq.~\ref{eq:app_spike}.

\subsection{Shared-prefix cross-evaluation protocol}

The fixed 4,000-update offline checkpoint generated once and stored one 32-token trajectory for each of the same 256 prompts, forming a common prefix bank. Offline-only, the 500-update Vanilla OPD checkpoint, and the 500-update \method{} checkpoint are cross-evaluated by replaying every stored token identity unchanged through all three checkpoints; no checkpoint-specific resampling is performed during the diagnostic comparison. The frozen ANN teacher also scores these same stored prefixes. For internal-dynamics metrics, each stored rollout prefix is paired with its corresponding offline-prefix control, and the same pair is used for all three checkpoints. We first compute teacher KL, spike-rate drift, hidden-state gap, and membrane drift within each prompt--prefix pair and then aggregate across the common bank. This matched-input protocol measures conditional checkpoint responses on a common prefix distribution; it does not compare independently resampled native rollouts. As stated in the main paper, these trajectory diagnostics are reported for the 125M model only.

\section{Why Full-Vocabulary KL Is an Exact Stage~2 Target}
\label{app:full_kl_theory}

\subsection{Snapshot objective and unbiased trajectory average}

Because the active policy changes between updates, distinguish the parameters $\bar\theta$ that generate a rollout batch from the parameters $\theta$ differentiated in that update. Define the snapshot objective
\begin{equation}
\begin{aligned}
J_{\mathrm{OPD}}(\theta;\bar\theta)
&=\E_{x\sim\mathcal{D}_{\mathrm{prompt}}}
\E_{y\sim\pi_{\bar\theta}(\cdot\mid x)}\left[
\frac{1}{K}\sum_{t=1}^{K}\right.\\[-2pt]
&\hspace{3.3em}\left.
\KL(p_T(\cdot\mid s_t)\|q_\theta(\cdot\mid s_t))
\right].
\end{aligned}
\label{eq:app_snapshot}
\end{equation}

\begin{proposition}[Exact per-prefix evaluation]
At any fixed prefix $s_t$, Eq.~\ref{eq:app_opd} evaluates the teacher--student forward KL exactly over $\V$. It introduces no bias from vocabulary truncation or sampled teacher tokens.
\end{proposition}
\begin{proof}
Both $p_T(\cdot\mid s_t)$ and $q_\theta(\cdot\mid s_t)$ are explicitly evaluated at every vocabulary coordinate, and the implemented loss is the defining finite sum
\begin{equation}
\sum_{v\in\V}p_T(v\mid s_t)
\log\frac{p_T(v\mid s_t)}{q_\theta(v\mid s_t)}.
\end{equation}
No random token subset appears in this expression.
\end{proof}

\begin{proposition}[Unbiased Monte Carlo estimate of the snapshot objective]
Let $B$ prompt--trajectory pairs be sampled independently from the distribution in Eq.~\ref{eq:app_snapshot}. Their empirical mean is unbiased for $J_{\mathrm{OPD}}(\theta;\bar\theta)$. Under the usual interchangeability condition for differentiation and expectation, its detached-prefix gradient is unbiased for $\nabla_\theta J_{\mathrm{OPD}}(\theta;\bar\theta)$.
\end{proposition}
\begin{proof}
The result follows from linearity of expectation over independently sampled prompt--trajectory pairs. Because rollout identities are detached, differentiation acts only on $q_\theta$ at the realized prefixes. If the per-trajectory gradient is integrable, differentiation and expectation can be interchanged.
\end{proof}

The statement is deliberately about the snapshot, or semi-gradient, objective. It does not differentiate through discrete sampling or through the dependence of the future state distribution on $\theta$, and it is not a global convergence claim for the nonconvex SNN. The prefix sample is stochastic; the vocabulary sum at each sampled prefix is exact.

\subsection{Cross-entropy equivalence, gradient, and curvature}

At a fixed prefix, abbreviate $p_v=p_T(v\mid s_t)$ and $q_v=q_\theta(v\mid s_t)$. Then
\begin{equation}
\KL(p\|q)=-\sum_{v\in\V}p_v\log q_v-H(p),
\label{eq:app_ce_equivalence}
\end{equation}
where the teacher entropy $H(p)$ is constant with respect to the student. Full forward KL and full-vocabulary teacher cross-entropy therefore have identical student gradients.

Let $z\in\mathbb{R}^{|\V|}$ be the student logits and $q=\operatorname{softmax}(z)$. For every token $j$,
\begin{equation}
\frac{\partial\KL(p\|q)}{\partial z_j}=q_j-p_j.
\label{eq:app_logit_gradient}
\end{equation}
Thus every vocabulary coordinate receives a calibrated correction. Tokens overweighted by the student have positive gradient, while tokens underweighted relative to the teacher have negative gradient.

\begin{proposition}[Positive semidefinite logit curvature]
The Hessian of the per-prefix forward KL with respect to student logits is
\begin{equation}
\nabla_z^2\KL(p\|q)=\operatorname{diag}(q)-qq^\top\succeq0.
\label{eq:app_hessian}
\end{equation}
Hence the loss is convex in logits, up to the softmax-invariant additive logit direction.
\end{proposition}
\begin{proof}
For any vector $a$,
\begin{equation}
\begin{aligned}
a^\top(\operatorname{diag}(q)-qq^\top)a
&=\E_{j\sim q}[a_j^2]-\E_{j\sim q}[a_j]^2\\
&=\operatorname{Var}_{j\sim q}(a_j)\ge0.
\end{aligned}
\end{equation}
\end{proof}
This local logit property does not make the parameterized SNN objective convex. Its practical value is that full KL supplies a dense, well-defined correction on every visited prefix without introducing an additional support-selection rule.

\section{Stability Terms and Their Gradients}
\label{app:stability}

\subsection{Matched-prefix reference KL}

At one prefix, abbreviate $q=q_\theta(\cdot\mid s_t)$ and $r=q_{\mathrm{ref}}(\cdot\mid s_t)$. The reference term satisfies
\begin{equation}
\KL(q\|r)=\E_{j\sim q}[\log q_j-\log r_j].
\label{eq:app_ref_identity}
\end{equation}
It therefore measures the expected log-probability displacement from the frozen offline SNN under behavior currently favored by the active student. For student logit $z_j$,
\begin{equation}
\frac{\partial\KL(q\|r)}{\partial z_j}
=q_j\left(\log\frac{q_j}{r_j}-\KL(q\|r)\right).
\label{eq:app_ref_gradient}
\end{equation}
The teacher and reference play different roles on the same prefix: Eq.~\ref{eq:app_logit_gradient} points toward the ANN teacher, whereas Eq.~\ref{eq:app_ref_gradient} penalizes a large departure from the competent offline policy. Because this is a soft penalty, teacher-supported deviations remain possible when their reduction in full-KL loss outweighs the reference cost.

Using the active distribution in the first argument is intentional. It emphasizes tokens to which the updated student assigns probability even when the offline reference considers them unlikely. This is the behavior implicated in self-reinforcing rollout drift and repetition collapse. The empirical stability results in the main paper support this mechanism; Eq.~\ref{eq:app_ref_identity} alone is not a proof that every trajectory remains stable.

\subsection{Spike-rate regularization}

For one layer, define $N_r=BKT_sN_l$, so that $r=N_r^{-1}\sum_i s_i$, where the compact index $i$ ranges over batch items, response positions, simulation steps, and neurons. The derivative of the layer penalty in Eq.~\ref{eq:app_spike} with respect to an individual spike (with derivative zero at an interval boundary) is
\begin{equation}
\begin{aligned}
\frac{\partial\ell_l}{\partial s_i}
&=\frac{2}{N_r}\bigl[
(r-r_{\min})\mathbf{1}_{r<r_{\min}}\\
&\quad +(r-r_{\max})\mathbf{1}_{r>r_{\max}}
+\rho(r-r^{\mathrm{ref}})\bigr].
\end{aligned}
\label{eq:app_spike_gradient}
\end{equation}
Binary thresholding is differentiated with a surrogate derivative,
$\partial s_i/\partial u_i\approx\psi(u_i-V_{\mathrm{th}})$ \citep{neftci2019surrogate}.

\begin{proposition}[Bounded activity gradient]
Suppose all rates and interval endpoints lie in $[0,1]$ and $|\psi|\le M_\psi$. The magnitude of the membrane-level gradient contributed by one monitored layer and one spike is at most
\begin{equation}
\frac{2(1+\rho)M_\psi}{N_r},
\end{equation}
before the outer average over the monitored set (four layers in the default configuration) and the fixed factor $0.30$.
\end{proposition}
\begin{proof}
At most one interval term is active, and its absolute deviation is at most one. The reference-rate deviation is also at most one. Apply the triangle inequality to Eq.~\ref{eq:app_spike_gradient} and multiply by the surrogate-gradient bound.
\end{proof}

The interval terms discourage silent or saturated operating regimes, whereas the reference-rate term retains how activity is distributed across depth. A network-wide mean would allow one overactive layer to cancel an underactive layer; the layerwise loss does not. The constraint acts on aggregate event traffic, not on individual spike identities, hidden states, or membrane values.

\section{Analytical Efficiency Reporting}
\label{app:energy}

We follow the analytical inference protocol used by \bispikclm{} \citep{guo2026bispikclm}. This is an operation-based estimate, not a wall-plug or on-chip power measurement. Under the 45-nm operation-energy model adopted by that protocol, a dense multiply--accumulate (MAC) costs
\begin{equation}
E_{\mathrm{MAC}}=4.6\ \mathrm{pJ},
\end{equation}
whereas an event-driven accumulation (AC) costs
\begin{equation}
E_{\mathrm{AC}}=0.9\ \mathrm{pJ}.
\end{equation}

For a spike-driven component $c$ in layer $l$, let $f_r^{c}(l)$ be its measured firing rate in spikes per neuron per simulation step, $T_s$ the number of simulation steps, and $\mathrm{FLOPs}_{c}(l)$ the MAC count of the equivalent dense component. Its synaptic-operation count is
\begin{equation}
\mathrm{SOPs}_{c}(l)=f_r^{c}(l)T_s\mathrm{FLOPs}_{c}(l).
\label{eq:app_sops}
\end{equation}
Embedding and LM-head computations remain dense, while the SFSA and SFFN blocks are counted as spike-driven AC computation. The reported analytical operation count is
\begin{equation}
\begin{aligned}
\mathrm{OPs}_{\mathrm{SNN}}
&=\mathrm{FLOPs}_{\mathrm{Embed}}
+\mathrm{FLOPs}_{\mathrm{LM\mbox{-}head}}\\
&\quad+\sum_{l=1}^{L}\left(
\mathrm{SOPs}_{\mathrm{SFSA}}(l)
+\mathrm{SOPs}_{\mathrm{SFFN}}(l)\right),
\end{aligned}
\label{eq:app_snn_ops}
\end{equation}
and the corresponding energy estimate is
\begin{equation}
\begin{aligned}
E_{\mathrm{SNN}}
&=E_{\mathrm{MAC}}
\left(\mathrm{FLOPs}_{\mathrm{Embed}}+\mathrm{FLOPs}_{\mathrm{LM\mbox{-}head}}\right)\\
&\quad+E_{\mathrm{AC}}\sum_{l=1}^{L}
\left(\mathrm{SOPs}_{\mathrm{SFSA}}(l)
+\mathrm{SOPs}_{\mathrm{SFFN}}(l)\right).
\end{aligned}
\label{eq:app_snn_energy}
\end{equation}
For the matched ANN teacher, all counted components use dense MACs, giving
\begin{equation}
E_{\mathrm{ANN}}=E_{\mathrm{MAC}}\,\mathrm{FLOPs}_{\mathrm{ANN,total}}.
\label{eq:app_ann_energy}
\end{equation}
When operation counts are expressed in billions, multiplying by pJ per operation yields mJ. The main table reports the aggregate model Rate rounded to two decimal places, but Eqs.~\ref{eq:app_sops}--\ref{eq:app_snn_energy} use the unrounded component- and layer-specific firing rates. Consequently, the displayed aggregate Rate is not multiplied by a single FLOP total and cannot be used to reconstruct OPs or energy. In particular, a slightly higher rounded aggregate Rate can coexist with lower OPs when activity is redistributed toward components with smaller equivalent dense-operation counts. The OPs column also combines dense MAC counts and spike-driven AC counts, whose energy constants differ. The aggregate OPs, Rate, and Energy values are reported in the main paper; they should be interpreted as arithmetic proxies under this shared boundary, not as hardware measurements.

\section{Offline \spad{} Initialization}
\label{app:offline_spad}

This section specifies how the 4,000-update checkpoint $\theta_0$ is constructed. It expands the brief initialization description in the main paper into an implementation-level record; it does not introduce an additional contribution or a loss used by \method{}. Stage~1 adopts the public \bispikclm{} \spad{} data, optimizer, temporal, surrogate, and five-branch outer-loss recipe \citep{guo2026bispikclm}; the two inner SAA/SFA mixing coefficients are checkpoint-local fields and are identified separately below. The ANN teacher is frozen, the SNN student is trainable, and both use the OPT tokenizer and vocabulary. After Stage~1, \method{} discards the five losses below and accesses the checkpoint only through its initialization, $q_{\mathrm{ref}}$, and $r_l^{\mathrm{ref}}$; Continued \spad{} is the intentionally separate control that reuses them.

\subsection{Stage~1 data and optimization record}
\label{app:stage1_optimization}

The retained checkpoint uses the $T_s=4$ branch of the public recipe and the FineWeb-Edu 10BT sample without an additional filtering stage \citep{penedo2024fineweb}. ``10BT'' identifies the source pool, not the number of tokens necessarily consumed by the 4,000 retained updates. The public training record does not state a fixed sequence length or a packing policy, so we do not infer either quantity from the corpus name. Table~\ref{tab:app_stage1_reproducibility} records every Stage~1 field available in the public report or the retained checkpoint configuration and separates it from the Stage~2 settings in Table~\ref{tab:app_reproducibility}.

\begin{table}[t]
\centering
\small
\setlength{\tabcolsep}{4pt}
\begin{tabular}{ll}
\toprule
Stage~1 setting & Value \\
\midrule
Teacher / student & OPT / \bispikclm{} \\
Tokenizer and vocabulary & Shared with OPT \\
Updates used for $\theta_0$ & 4,000 \\
Simulation steps $T_s$ & 4 \\
Optimizer & Adam \\
Learning rate & $5\times10^{-4}$ \\
Schedule & Cosine decay \\
Warm-up ratio & 0.2 \\
Microbatch per GPU & 16 \\
Gradient accumulation & 16 steps \\
Effective batch per GPU & 256 \\
Global effective batch (8 GPUs) & 2,048 \\
Global gradient-norm clip & 0.7 \\
Soft-target temperature $T_{\mathrm{KD}}$ & 2.0 \\
Surrogate family / sharpness & Arctangent / $\xi=2$ \\
Hardware & $8\times$ RTX 4090 (24GB) \\
\bottomrule
\end{tabular}
\caption{Stage~1 initialization record. The public recipe's effective batch of 256 is device-local ($16\times16$); under eight-way data parallelism the corresponding global effective batch is 2,048 sequences per optimizer step.}
\label{tab:app_stage1_reproducibility}
\end{table}

\subsection{Implementation-level softmax-free spiking attention}
\label{app:stage1_sfsa}

Let $B_1$ be the offline minibatch size, $n$ the padded sequence length, $h$ the number of heads, and $d_h$ the width of one head. At student layer $l$ and simulation step $\tau$, the input is $X_S^{l,\tau}\in\{0,1\}^{B_1\times n\times hd_h}$. The compact main-paper expression names the pre-output aggregation $A_SV_S$ as $O$ and writes score scaling explicitly; the public executable path uses unscaled coincidence counts and continues through the downstream attention-output and output-projection neurons. These are hard-forward parameterizations of the same stateful attention neuron only when the complete membrane scale is transformed. Specifically, with $c=\sqrt{d_h}$, compare $u^\tau=\alpha u^{\tau-1}+C^\tau/c-v_Aa^{\tau-1}$ against $\widetilde u^\tau=\alpha\widetilde u^{\tau-1}+C^\tau-cv_Aa^{\tau-1}$. If $\widetilde u^0=cu^0$ and there is no additional unscaled bias, induction gives $\widetilde u^\tau=cu^\tau$ and therefore identical hard spikes under thresholds $v_A$ and $cv_A$. This observation reconciles the forward notation; it does not assert that the two surrogate backward graphs are identical at a fixed sharpness. All Stage~1 equations and gradient bounds below refer to the unscaled public execution path with $\xi=2$. For head $a$, that implementation-level order is
\begin{equation}
\begin{aligned}
Q_{S,a}^{l,\tau}
&=\mathcal S_Q(X_S^{l,\tau}W_{Q,a}^{l}),\\
K_{S,a}^{l,\tau}
&=\mathcal S_K(X_S^{l,\tau}W_{K,a}^{l}),\\
V_{S,a}^{l,\tau}
&=\mathcal S_V(X_S^{l,\tau}W_{V,a}^{l}),\\
C_{S,a}^{l,\tau}
&=Q_{S,a}^{l,\tau}(K_{S,a}^{l,\tau})^\top,\\
\widetilde C_{S,a}^{l,\tau}
&=M\odot C_{S,a}^{l,\tau},\\
A_{S,a}^{l,\tau}
&=M\odot\mathcal S_A(\widetilde C_{S,a}^{l,\tau}),\\
Z_{S,a}^{l,\tau}
&=A_{S,a}^{l,\tau}V_{S,a}^{l,\tau},\\
R_{S,a}^{l,\tau}
&=\mathcal S_{AV}(Z_{S,a}^{l,\tau}),\\
O_S^{l,\tau}
&=\mathcal S_O\!\left(
\operatorname{Concat}_{a=1}^{h}R_{S,a}^{l,\tau}W_O^l\right).
\end{aligned}
\label{eq:app_sfsa_detailed}
\end{equation}
For batch item $b$, let $m_{b,i}\in\{0,1\}$ be the non-padding indicator and define $M_{b,i,j}=m_{b,i}m_{b,j}\mathbf 1[j\le i]$. The first multiplication in Eq.~\ref{eq:app_sfsa_detailed} produces integer coincidence counts from binary $Q$ and $K$. Masked counts enter the attention neuron; the mask is reapplied to its output, and $\mathcal S_{AV}$ converts the integer $AV$ accumulation back to spikes before output projection. Consequently, $A_{S,a,b,i,j}^{l,\tau}=0$ for every future or padded key, and $R_{S,a,b,i}^{l,\tau}$ depends only on valid value spikes at positions $j\le i$. Equation~\ref{eq:app_sfsa_detailed}'s $Z_S=A_SV_S$ is the tensor denoted by $O$ in the main-paper shorthand; $R_S$ and the final projection spell out the subsequent public SFSA operations. There is no softmax, and the threshold parameterization avoids executing $1/\sqrt{d_h}$ as a separate floating-point operation.

For supervision, the frozen teacher exports post-softmax causal attention maps $A_{T,b}^{m}\in[0,1]^{n\times n}$. A layer map $\pi(l)$ and a head map $\chi_l(a)$ identify the teacher tensor paired with student layer $l$ and head $a$. Uniform layer skipping and head-wise mapping are used only when depths or head counts differ; for a shape-matched OPT--\bispikclm{} pair, both maps are the identity. Padding locations, future-token entries, and positions without a next-token target are excluded by their corresponding masks.

\subsection{Temporal fusion and teacher-to-spike encoding}
\label{app:stage1_operators}

For any student tensor $X_S^\tau$, temporal fusion is the empirical rate
\begin{equation}
\bar X_S=\frac{1}{T_s}\sum_{\tau=1}^{T_s}X_S^\tau.
\label{eq:app_temporal_fusion}
\end{equation}
For a set of valid scalar coordinates $\Omega_X$, we use
\begin{equation}
\operatorname{MSE}_{\Omega_X}(X,Y)
=\frac{1}{|\Omega_X|}\sum_{\omega\in\Omega_X}(X_\omega-Y_\omega)^2.
\label{eq:app_masked_mse}
\end{equation}
This definition fixes the normalization for padding and causal masks.

The static ANN tensor must also be represented in the student's spike domain. For each scalar $x$ of a teacher tensor $X_T$, repeat $x$ as a constant current and run the same subtractive-reset LIF dynamics:
\begin{equation}
\begin{aligned}
u_{x}^{\tau}&=\alpha u_{x}^{\tau-1}+x
-V_{\mathrm{th}}s_{x}^{\tau-1},\\
s_{x}^{\tau}&=\mathbf 1[u_x^\tau\ge V_{\mathrm{th}}],
\qquad
\mathcal R_{T_s}(x)=\frac{1}{T_s}\sum_{\tau=1}^{T_s}s_x^\tau.
\end{aligned}
\label{eq:app_teacher_spike_proxy}
\end{equation}
Equation~\ref{eq:app_teacher_spike_proxy} is applied entry-wise, with zero initial membrane state and the same neuron parameters as the corresponding student module. Thus $\mathcal R_{T_s}(X_T)$ is an empirical-rate target in $[0,1]$ rather than an unspecified quantizer. At $T_s=4$ it is a finite-horizon proxy; no asymptotic equality between ANN values and spike rates is assumed.

\subsection{Five aligned targets}
\label{app:stage1_losses}

For an offline sequence $x_{1:n}$, let the teacher export embeddings $E_T$, hidden states $H_T^m$, attention maps $A_T^m$, and logits $z_T$. The student produces $E_S^\tau$, $H_S^{l,\tau}$, $A_S^{l,\tau}$, and $z_S$. Projection $P_e$ maps student embeddings to teacher width. In the continuous feature branch, $P_h^l$ is the public MLP or linear width-matching map and is followed by LayerNorm. For a shape-matched pair the width projection may reduce to the identity, but LayerNorm remains part of the alignment transform.

\paragraph{Embedding alignment (EA).}
The first loss supplies a target before the first transformer block:
\begin{equation}
\loss_{\mathrm{EA}}
=\operatorname{MSE}_{\Omega_E}\!\left(P_e(\bar E_S),E_T\right).
\label{eq:app_ea}
\end{equation}

\paragraph{Spike-attention alignment (SAA).}
Let $\mathcal P_A$ contain every matched student-layer/head pair $(l,a)$, and let $\Omega_A$ contain causal, non-padding attention entries. The rate-domain and continuous-domain terms are
\begin{equation}
\begin{aligned}
\loss_{\mathrm{SAA}}^{\mathrm{rate}}
&=\frac{1}{|\mathcal P_A|}\sum_{(l,a)\in\mathcal P_A}
\operatorname{MSE}_{\Omega_A}\!\left(
\bar A_{S,a}^{l},
\mathcal R_{T_s}(A_{T,\chi_l(a)}^{\pi(l)})\right),\\
\loss_{\mathrm{SAA}}^{\mathrm{cont}}
&=\frac{1}{|\mathcal P_A|}\sum_{(l,a)\in\mathcal P_A}
\operatorname{MSE}_{\Omega_A}\!\left(
\bar A_{S,a}^{l},A_{T,\chi_l(a)}^{\pi(l)}\right),\\
\loss_{\mathrm{SAA}}
&=\gamma_{\mathrm{attn}}\loss_{\mathrm{SAA}}^{\mathrm{rate}}
+(1-\gamma_{\mathrm{attn}})\loss_{\mathrm{SAA}}^{\mathrm{cont}}.
\end{aligned}
\label{eq:app_saa}
\end{equation}
The first branch compares two empirical spike rates and therefore respects the student's discrete representation. The second branch preserves the teacher's continuous relational structure after student temporal fusion. Unlike the abbreviated earlier definition, Eq.~\ref{eq:app_saa} requires no row-renormalization operator, so an all-zero student attention row remains well-defined.

\paragraph{Spike-feature alignment (SFA).}
For the matched layer set $\mathcal P_H$, define
\begin{equation}
\widetilde H_S^l
=\operatorname{LayerNorm}\!\left(P_h^l(\bar H_S^l)\right).
\label{eq:app_projected_feature}
\end{equation}
The feature objectives are
\begin{equation}
\begin{aligned}
\loss_{\mathrm{SFA}}^{\mathrm{rate}}
&=\frac{1}{|\mathcal P_H|}\sum_{l\in\mathcal P_H}
\operatorname{MSE}_{\Omega_H}\!\left(
\bar H_S^l,\mathcal R_{T_s}(H_T^{\pi(l)})\right),\\
\loss_{\mathrm{SFA}}^{\mathrm{cont}}
&=\frac{1}{|\mathcal P_H|}\sum_{l\in\mathcal P_H}
\operatorname{MSE}_{\Omega_H}\!\left(
\widetilde H_S^l,H_T^{\pi(l)}\right),\\
\loss_{\mathrm{SFA}}
&=\gamma_{\mathrm{feat}}\loss_{\mathrm{SFA}}^{\mathrm{rate}}
+(1-\gamma_{\mathrm{feat}})\loss_{\mathrm{SFA}}^{\mathrm{cont}}.
\end{aligned}
\label{eq:app_sfa}
\end{equation}
When widths differ, the teacher tensor in the rate branch is first mapped to student width and the student tensor in the continuous branch is mapped to teacher width. In the models used here the widths match, so these shape maps reduce to identities and do not add deployment parameters.

\paragraph{Soft- and hard-token alignment.}
For valid next-token positions $\Omega_{\mathrm{tok}}$, define
\begin{equation}
\begin{aligned}
p_T^{t,T_{\mathrm{KD}}}
&=\operatorname{softmax}(z_T^t/T_{\mathrm{KD}}),\\
q_S^{t,T_{\mathrm{KD}}}
&=\operatorname{softmax}(z_S^t/T_{\mathrm{KD}}),\\
\loss_{\mathrm{STA}}
&=\frac{T_{\mathrm{KD}}^2}{|\Omega_{\mathrm{tok}}|}
\sum_{t\in\Omega_{\mathrm{tok}}}
\KL\!\left(p_T^{t,T_{\mathrm{KD}}}\|q_S^{t,T_{\mathrm{KD}}}\right),\\
\loss_{\mathrm{HTA}}
&=-\frac{1}{|\Omega_{\mathrm{tok}}|}
\sum_{t\in\Omega_{\mathrm{tok}}}
\log q_S(x_{t+1}\mid x_{\le t}).
\end{aligned}
\label{eq:app_token_alignment}
\end{equation}
The factor $T_{\mathrm{KD}}^2$ compensates for the $1/T_{\mathrm{KD}}$ factor introduced when differentiating the softened logits.

\paragraph{Complete Stage~1 objective.}
The checkpoint is trained with
\begin{equation}
\begin{aligned}
\loss_{\spad}
&=\lambda_{\mathrm{emb}}\loss_{\mathrm{EA}}
+\lambda_{\mathrm{attn}}\loss_{\mathrm{SAA}}
+\lambda_{\mathrm{feat}}\loss_{\mathrm{SFA}}\\
&\quad+\lambda_{\mathrm{soft}}\loss_{\mathrm{STA}}
+\lambda_{\mathrm{hard}}\loss_{\mathrm{HTA}},
\end{aligned}
\label{eq:app_offline_total}
\end{equation}
where
\begin{equation}
\begin{gathered}
T_{\mathrm{KD}}=2,\qquad
(\lambda_{\mathrm{emb}},\lambda_{\mathrm{attn}},\lambda_{\mathrm{feat}},
\lambda_{\mathrm{soft}},\lambda_{\mathrm{hard}})\\
=(0.2,0.1,0.1,0.3,0.3),\qquad
\gamma_{\mathrm{attn}}=\gamma_{\mathrm{feat}}=0.5.
\end{gathered}
\label{eq:app_offline_weights}
\end{equation}
The five outer weights sum to one. This keeps the numerical loss scale interpretable, but does not by itself guarantee equal parameter-gradient norms across the five branches. The two inner mixing coefficients are checkpoint-configuration fields recorded for this submission; they are not reported in the public training table and therefore constitute a checkpoint-specific completion of the public outer recipe, not a value attributed to the public paper. These Stage~1 coefficients must not be confused with the Stage~2 ratio $1:0.75:0.3$ in Eq.~\ref{eq:app_total}.

\subsection{Stage~1 update sequence}
\label{app:stage1_update}

Each offline update has the following order:
\begin{enumerate}
    \item tokenize a FineWeb-Edu batch once and use the same padding and next-token masks for teacher and student;
    \item run the frozen OPT teacher without gradient storage and cache the matched embeddings, attention maps, hidden states, and logits;
    \item unroll the \bispikclm{} student for $T_s=4$ simulation steps, retaining the spike and membrane traces required by BPTT;
    \item construct the teacher spike proxies in Eq.~\ref{eq:app_teacher_spike_proxy}, then evaluate Eqs.~\ref{eq:app_ea}, \ref{eq:app_saa}, \ref{eq:app_sfa}, and~\ref{eq:app_token_alignment};
    \item form Eq.~\ref{eq:app_offline_total}, backpropagate only through the student and any training-time alignment projections, clip the resulting global gradient norm to $0.7$, and take one Adam update under the cosine schedule.
\end{enumerate}
Only the student parameters are retained in $\theta_0$. Teacher-side spike proxies and width-alignment projections are training interfaces; they do not create an auxiliary inference path.

\subsection{Surrogate-gradient propagation with subtractive reset}
\label{app:stage1_gradient_stability}

The main paper states the LIF forward dynamics. Here we analyze the backward recurrence actually induced by its subtractive reset. To avoid overloading the leak $\alpha$, let $\xi$ denote the sharpness of the arctangent surrogate:
\begin{equation}
\begin{aligned}
\widetilde{\mathcal S}_{\xi}(v)
&=\frac{1}{\pi}\arctan\!\left(\frac{\pi\xi v}{2}\right)
+\frac{1}{2},\\
\psi_\xi(v)
&=\frac{\xi/2}{1+(\pi\xi v/2)^2}.
\end{aligned}
\label{eq:app_arctan_surrogate}
\end{equation}
Stage~1 uses $\xi=2$, and therefore
\begin{equation}
0\le\psi_\xi(v)\le M_\psi=\xi/2=1.
\label{eq:app_surrogate_bound}
\end{equation}
The hard threshold is used in the forward pass; Eq.~\ref{eq:app_arctan_surrogate} supplies only the backward derivative.

Consider one vector-valued LIF module with input current $I^\tau=WX^\tau$:
\begin{equation}
u^\tau=\alpha u^{\tau-1}+WX^\tau
-V_{\mathrm{th}}s^{\tau-1},
\qquad
s^\tau=\mathbf 1[u^\tau\ge V_{\mathrm{th}}].
\label{eq:app_lif_local}
\end{equation}
Holding the presynaptic trace fixed for this local temporal calculation, define
\begin{equation}
D_{\tau}=\operatorname{diag}\!\left(
\psi_\xi(u^\tau-V_{\mathrm{th}})\right),
\qquad
J_\tau=\alpha I-V_{\mathrm{th}}D_{\tau-1}.
\label{eq:app_temporal_jacobian}
\end{equation}
Then the surrogate temporal Jacobian is
\begin{equation}
\frac{\partial u^\tau}{\partial u^k}
\approx J_\tau J_{\tau-1}\cdots J_{k+1},
\qquad 1\le k<\tau\le T_s.
\label{eq:app_temporal_product}
\end{equation}
This expression retains the derivative of the reset spike. Dropping that term would replace $J_\tau$ by $\alpha I$ and would not describe Eq.~\ref{eq:app_lif_local}.

\begin{proposition}[Finite-horizon temporal-gradient bound]
\label{prop:app_temporal_bound}
Let $\bar\kappa=\max_{2\le\tau\le T_s}\|J_\tau\|_2$. Under the surrogate derivative in Eq.~\ref{eq:app_arctan_surrogate},
\begin{equation}
\left\|\frac{\partial u^\tau}{\partial u^k}\right\|_2
\le\bar\kappa^{\tau-k},
\qquad
\bar\kappa\le
\max\!\left\{\alpha,\left|\alpha-V_{\mathrm{th}}M_\psi\right|\right\}.
\label{eq:app_temporal_bound}
\end{equation}
If $\bar\kappa<1$, the temporal state map is contractive in this local surrogate sense. Regardless of whether this sufficient condition holds, the maximum temporal amplification within the checkpoint used here is finite and bounded by
$\max(1,\bar\kappa^3)$ because $T_s=4$.
\end{proposition}
\begin{proof}
Each $D_\tau$ is diagonal with entries in $[0,M_\psi]$. Hence $J_\tau$ is diagonal with entries in
$[\alpha-V_{\mathrm{th}}M_\psi,\alpha]$, which gives the second inequality. Submultiplicativity of the spectral norm applied to Eq.~\ref{eq:app_temporal_product} gives the first. The largest separation between two of four simulation steps is three.
\end{proof}

The same Jacobian appears in the reset-aware eligibility trace. With
$e^\tau=\partial u^\tau/\partial W$ and
$\mathcal X^\tau=\partial(WX^\tau)/\partial W$, the local recurrence is
\begin{equation}
e^\tau=\mathcal X^\tau+J_\tau e^{\tau-1},
\qquad e^0=0.
\label{eq:app_eligibility}
\end{equation}

\begin{proposition}[Bounded reset-aware eligibility trace]
\label{prop:app_eligibility}
Assume $\|\mathcal X^\tau\|_2\le M_X$ and use the $\bar\kappa$ of Proposition~\ref{prop:app_temporal_bound}. Then
\begin{equation}
\|e^\tau\|_2
\le M_X\sum_{j=0}^{\tau-1}\bar\kappa^j.
\label{eq:app_eligibility_bound}
\end{equation}
For $\bar\kappa<1$, this is at most $M_X/(1-\bar\kappa)$; for the implemented $T_s=4$, it is at most
$M_X(1+\bar\kappa+\bar\kappa^2+\bar\kappa^3)$ without a contraction assumption.
\end{proposition}
\begin{proof}
Unroll Eq.~\ref{eq:app_eligibility}, apply the triangle inequality, and bound every product of $j$ temporal Jacobians by $\bar\kappa^j$.
\end{proof}

To include all future temporal paths, let $c^\tau$ denote the direct derivative of the loss component evaluated at step $\tau$ with respect to $s^\tau$, and let $a^\tau=\partial\loss/\partial u^\tau$ be the total membrane adjoint. BPTT gives
\begin{equation}
\begin{aligned}
a^\tau
&=D_\tau c^\tau+J_{\tau+1}^\top a^{\tau+1},
\qquad a^{T_s+1}=0,\\
\frac{\partial\loss}{\partial W}
&=\sum_{\tau=1}^{T_s}(\mathcal X^\tau)^\top a^\tau.
\end{aligned}
\label{eq:app_adjoint_recursion}
\end{equation}
If $\|c^\tau\|_2\le M_L$ and $\|\mathcal X^\tau\|_2\le M_X$, unrolling Eq.~\ref{eq:app_adjoint_recursion} yields
\begin{equation}
\begin{aligned}
\|a^\tau\|_2
&\le M_\psi M_L
\sum_{j=0}^{T_s-\tau}\bar\kappa^j,\\
\left\|\frac{\partial\loss}{\partial W}\right\|_2
&\le M_XM_\psi M_L
\sum_{\tau=1}^{T_s}
\sum_{j=0}^{T_s-\tau}\bar\kappa^j.
\end{aligned}
\label{eq:app_parameter_gradient_bound}
\end{equation}
This is a local temporal bound for the unrolled module before accounting for spatial Jacobians in preceding layers. The output losses are also bounded at the logit interface: for one valid position,
\begin{equation}
\begin{aligned}
&\left\|\frac{\partial}{\partial z_S}
\left[T_{\mathrm{KD}}^2
\KL(p_T^{T_{\mathrm{KD}}}\|q_S^{T_{\mathrm{KD}}})\right]
\right\|_2\\
&\qquad=T_{\mathrm{KD}}
\|q_S^{T_{\mathrm{KD}}}-p_T^{T_{\mathrm{KD}}}\|_2
\le\sqrt{2}\,T_{\mathrm{KD}},
\end{aligned}
\label{eq:app_sta_logit_bound}
\end{equation}
and hard-label cross-entropy satisfies
$\|\partial\operatorname{CE}/\partial z_S\|_2\le\sqrt{2}$.

These bounds isolate three safeguards: the surrogate slope is bounded, only four temporal steps are unrolled, and the realized global parameter gradient is clipped at $0.7$. They do not prove global convergence, rule out vanishing gradients, or bound products of all spatial layer Jacobians. They also concern the surrogate backward graph, not the nonexistent classical derivative of the hard threshold. Their purpose is narrower: to make explicit why the Stage~1 temporal recurrence has a controlled finite-horizon gradient contribution under stated bounded-input assumptions.

\section{Reproducibility Checklist}
\label{app:checklist}

Stage~1 follows the fixed record in Table~\ref{tab:app_stage1_reproducibility}: the frozen OPT teacher and trainable SNN share tokenization, all alignment masks are explicit, the arctangent surrogate uses $\xi=2$, and the global gradient norm is clipped at $0.7$. Stage~2 computes teacher, active-SNN, and reference-SNN logits on identical detached prefixes. It uses full-vocabulary KL with a stable log-softmax, $T_s=4$, and $\rho=1$; the default training layer set is $\mathcal S_{\mathrm{spk}}=\{3,6,9,12\}$, with only the explicit layer-sensitivity study changing it. In contrast, trajectory diagnostics average all $L$ layers at 125M and replay identical stored tokens through Offline-only, the 500-update Vanilla OPD checkpoint, and the 500-update \method{} checkpoint. The scale study uses three independent runs per scale and reports three-run means for all \method{} task accuracies, Avg., and firing rate. The collapse comparison uses the same 10 matched seeds for Vanilla OPD, Small-LR OPD, Clipped OPD, and \method{}; explicitly identified non-default sensitivity settings and other controlled ablations remain single runs. OPs and energy follow Eqs.~\ref{eq:app_sops}--\ref{eq:app_snn_energy} and are not hardware measurements.

\end{document}